%% file: adaptive-ucb.tex
\documentclass{article} %

\usepackage{etoolbox}
\newtoggle{preprint}
\newtoggle{icml}
\newtoggle{neurips}
\newcommand{\noticml}[1]{\iftoggle{icml}{}{#1}}

\newcommand{\icml}[1]{\iftoggle{icml}{#1}{}}

\newcommand{\preprint}[1]{\iftoggle{preprint}{#1}{}}
\newcommand{\neurips}[1]{\iftoggle{neurips}{#1}{}}

\toggletrue{preprint} %
\togglefalse{neurips} %
\togglefalse{icml}

\preprint{
\pdfoutput=1
\usepackage[T1]{fontenc}

\usepackage{lmodern} \normalfont %
\usepackage{anyfontsize}
\DeclareFontShape{T1}{lmr}{bx}{sc} { <-> ssub * cmr/bx/sc }{}
\DeclareFontShape{T1}{lmr}{m}{scit}{ <-> ssub * cmr/m/sc }{}
\DeclareFontShape{T1}{lmr}{bx}{scit}{ <-> ssub * cmr/bx/sc }{}
}

\include{header-files/global-macros}

\include{header-files/local-macros}

\preprint{
\title{Adaptively Exploiting \unobservedNames{} with Causal Bandits}
\author{Blair Bilodeau\footnote{Correspondence to \url{blair.bilodeau[at]mail.utoronto.ca}}\,, Linbo Wang, and Daniel M. Roy\\University of Toronto, Department of Statistical Sciences}
\date{}

\setlength{\parindent}{0pt}
\setlength{\parskip}{6pt}
\usepackage{titlesec}
\newcommand{\addperiod}[1]{#1.}
\titleformat{\paragraph}[runin]
  {\normalfont\normalsize\bfseries}
  {\theparagraph}
  {1em}
  {\addperiod}
\titlespacing{\section}{0pt}{\parskip}{0pt}
\titlespacing{\subsection}{0pt}{\parskip}{0pt}
\titlespacing{\subsubsection}{0pt}{\parskip}{0pt}
\usepackage[letterpaper, margin=1in]{geometry}
\usepackage[hide=true,setmargin=true,marginparwidth=1in]{marginalia}
}

\neurips{
\usepackage[final]{neurips_2022}
\title{Adaptively Exploiting \unobservedNames{}\\ with Causal Bandits}
\author{%
  Blair Bilodeau\thanks{Correspondence to blair.bilodeau@mail.utoronto.ca} \\
  University of Toronto \\
  \And
  Linbo Wang \\
  University of Toronto \\
  \And
  Daniel M. Roy \\
  University of Toronto \\
}
}

\icml{

\usepackage[accepted]{icml2022}
\icmltitlerunning{Adaptively Exploiting \unobservedNames{} with Causal Bandits}
}

\newcommand{\manualendproof}{\hfill\qedsymbol\\[2mm]}

\begin{document}

\icml{
\twocolumn[
\icmltitle{Adaptively Exploiting \unobservedNames{} with Causal Bandits}

\icmlsetsymbol{equal}{*}

\begin{icmlauthorlist}
\icmlauthor{Blair Bilodeau}{toronto,vector}
\icmlauthor{Linbo Wang}{toronto}
\icmlauthor{Daniel M. Roy}{toronto,vector}
\end{icmlauthorlist}

\icmlaffiliation{toronto}{Department of Statistical Sciences, University of Toronto}
\icmlaffiliation{vector}{Vector Institute}

\icmlcorrespondingauthor{Blair Bilodeau}{blair.bilodeau[at]mail.utoronto.ca}

\icmlkeywords{Bandits, Sequential Causal, Adaptive, Confounding}

\vskip 0.3in
]

\printAffiliationsAndNotice{}  %
}

\preprint{
\maketitle
\vspace{-30pt} %
}
\neurips{\maketitle}
\begin{abstract}

\input{abstract}
\end{abstract}

\input{section-files/main-sections/introduction}

\input{section-files/main-sections/bandit-defs}
\input{section-files/main-sections/causal-theorems}

\input{section-files/main-sections/main-theorems}

\input{section-files/main-sections/simulations}
\input{section-files/main-sections/adaptivity}

\input{section-files/main-sections/literature}

\input{section-files/main-sections/discussion}

\input{acknowledgements}

\bibliographystyle{abbrvnat}
\bibliography{bib-files/adaptive-ucb}

\neurips{\input{checklist}}

\newpage
\icml{\onecolumn}
\appendix

\input{section-files/main-sections/causal-proofs}

\input{section-files/main-sections/main-proofs}

\input{section-files/main-sections/adaptive-proof}

\input{section-files/main-sections/simulation-details}
\end{document}

%% file: header-files/global-macros.tex
\usepackage{
	amsmath,
	amsthm,
	amssymb,
	wrapfig,
	cases,
	mathtools,
	thmtools,
	array,
	bbm,
	bm,
	subfigure,
	makecell,
	esvect,
	mathrsfs,
	breqn,
	booktabs,
	upgreek,
	changepage,
	dsfont,
	fixme,
	listings,
	multirow,
	xargs,
	xstring,
	multicol,
	graphicx,
	float,
	color,
	enumerate,
	indentfirst,
	ifthen,
	wasysym,
	tikz,
	pgf,
	lmodern,
	titlesec,
	microtype
	}

\usepackage[utf8]{inputenc}
\usepackage[noend,ruled,vlined]{algorithm2e}
\usepackage{xpatch}
\makeatletter
\xpatchcmd{\algorithmic}
  {\ALG@tlm\z@}{\leftmargin\z@\ALG@tlm\z@}
  {}{}
\makeatother

\usepackage[hyperfootnotes=false, hidelinks]{hyperref}

\usepackage{natbib}
\setcitestyle{numbers,square,comma}

\usetikzlibrary{shapes,decorations,arrows,calc,arrows.meta,fit,positioning}
\tikzset{
    -Latex,auto,node distance =1 cm and 1 cm,semithick,
    state/.style ={ellipse, draw, minimum width = 0.7 cm},
    point/.style = {circle, draw, inner sep=0.04cm,fill,node contents={}},
    bidirected/.style={Latex-Latex,dashed},
    el/.style = {inner sep=2pt, align=left, sloped}
}

\usepackage[capitalize,nameinlink]{cleveref}
\usepackage{crossreftools}
\hypersetup{%
    bookmarksnumbered, bookmarksopen=true, bookmarksopenlevel=1,%
}

\crefname{figure}{Figure}{Figures}
\crefname{subsection}{Subsection}{Subsections}
\crefname{lemma}{Lemma}{Lemmas}
\crefname{corollary}{Corollary}{Corollaries}
\crefname{theorem}{Theorem}{Theorems}

\declaretheorem[name=Theorem]{theorem}

\declaretheorem[sibling=theorem,name=Lemma]{lemma}

\declaretheorem[sibling=theorem,name=Definition]{definition}

\declaretheorem[name=Assumption, numbered=no]{assumption*}

\declaretheorem[qed=$\triangleleft$,sibling=theorem,name=Remark]{remark}

\numberwithin{equation}{section}
\numberwithin{theorem}{section}
\numberwithin{lemma}{section}
\numberwithin{remark}{section}
\numberwithin{proposition}{section}
\numberwithin{definition}{section}

\makeatletter
\renewcommand{\maketag@@@}[1]{\hbox{\m@th\normalsize\normalfont#1}}%
\makeatother

\makeatletter
\let\reftagform@=\tagform@
\def\tagform@#1{\maketag@@@{\ignorespaces\textcolor{gray}{(#1)}\unskip\@@italiccorr}}
\renewcommand{\eqref}[1]{\textup{\reftagform@{\ref{#1}}}}
\makeatother

\renewcommand{\AA}{\mathbb{A}}

\newcommand{\EE}{\mathbb{E}}

\newcommand{\II}{\mathbb{I}}

\newcommand{\PP}{\mathbb{P}}

\newcommand{\RR}{\mathbb{R}}

\newcommand{\TT}{\mathbb{T}}

\newcommand{\Aa}{\mathcal{A}}
\newcommand{\Bb}{\mathcal{B}}

\newcommand{\Gg}{\mathcal{G}}

\newcommand{\Oo}{\mathcal{O}}

\newcommand{\Rr}{\mathcal{R}}

\newcommand{\Vv}{\mathcal{V}}

\newcommand{\Yy}{\mathcal{Y}}
\newcommand{\Zz}{\mathcal{Z}}

\newcommand{\eps}{\varepsilon}

\def\[#1\]{\begin{equation}\begin{aligned}#1\end{aligned}\end{equation}}
\def\*[#1\]{\begin{equation*}\begin{aligned}#1\end{aligned}\end{equation*}}

\def\s*[#1\s]{\small\begin{align*}#1\end{align*}\normalsize}

\newcommand{\lcrx}[4][{-1}]{ 
	\IfEq{#1}{-1}{\left #2 {{{{#3}}}} \right #4}{
   	\IfEq{#1}{0}{#2 {{{{#3}}}} #4}{
	\IfEq{#1}{1}{\bigl #2 {{{{#3}}}} \bigr #4}{
	\IfEq{#1}{2}{\Bigl #2 {{{{#3}}}} \Bigr #4}{
	\IfEq{#1}{3}{\biggl #2 {{{{#3}}}} \biggr #4}{
	\IfEq{#1}{4}{\Biggl #2 {{{{#3}}}} \Biggr #4}{
    \GenericWarning{"4th argument to lcrx must be -1, 0, 1, 2, 3, or 4"}
    }}}}}}} %

\newcommand{\KL}[2]{\mathrm{KL}\rbra{#1 \ \Vert \ #2}}

\newcommand{\setdelim}{\ \vert \ }
\newcommand{\Bigsetdelim}{\ \Big\vert \ }

\newcommand{\iid}{\text{i.i.d.}}

\newcommand{\ind}{\II} %

\def\multiset#1#2{\ensuremath{\left(\kern-.3em\left(\genfrac{}{}{0pt}{}{#1}{#2}\right)\kern-.3em\right)}}

\DeclareMathOperator*{\argmax}{\arg\max} %
\DeclareMathOperator*{\newlim}{\mathrm{lim}\vphantom{\mathrm{infsup}}}
\DeclareMathOperator*{\newmin}{\mathrm{min}\vphantom{\mathrm{infsup}}}
\DeclareMathOperator*{\newmax}{\mathrm{max}\vphantom{\mathrm{infsup}}}
\DeclareMathOperator*{\newinf}{\mathrm{inf}\vphantom{\mathrm{infsup}}}
\DeclareMathOperator*{\newsup}{\mathrm{sup}\vphantom{\mathrm{infsup}}}
\renewcommand{\lim}{\newlim}
\renewcommand{\min}{\newmin}
\renewcommand{\max}{\newmax}
\renewcommand{\inf}{\newinf}
\renewcommand{\sup}{\newsup}

\newcommand{\dee}{\mathrm{d}} %

\newcommand{\bernoullidist}{\mathrm{Ber}}

\newcommand{\rbra}[2][{-1}]{\lcrx[#1] ( {#2} ) }

\newcommand{\sbra}[2][{-1}]{\lcrx[#1] [ {#2} ] }

\newcommand{\abs}[2][{-1}]{\lcrx[#1] \vert {#2} \vert }

\newcommand{\defas}{:=}

\newcommand{\Reals}{\RR}

\newcommand{\range}[2][{1}]{
	\IfEq{#1}{1}{\sbra{#2}}{\sbra{#2}_{#1}}}
\newcommand{\rangeO}[2][{0}]{
	\IfEq{#1}{0}{\sbra{#2}_0}{\sbra{#2}_{#1}}}

%% file: header-files/local-macros.tex
\newcommand{\statname}{experimenter}

\newcommand{\unobservedname}{$d$-separator}

\newcommand{\unobservednames}{$d$-separators}
\newcommand{\unobservedNames}{$d$-Separators}

\newcommand{\nullaction}{null}
\newcommand{\genactionname}{intervention}
\newcommand{\genactionnames}{interventions}

\newcommand{\actionname}{action}
\newcommand{\actionnames}{actions}

\newcommand{\postcontextname}{post-action context}
\newcommand{\postcontextnames}{post-action contexts}

\newcommand{\envname}{environment}
\newcommand{\envnames}{environments}
\newcommand{\envName}{Environment}

\newcommand{\policyname}{policy}
\newcommand{\policynames}{policies}

\newcommand{\algoname}{algorithm}
\newcommand{\algonames}{algorithms}

\newcommand{\playername}{player}

\newcommand{\historyname}{observed history}

\newcommand{\regretname}{regret}

\newcommand{\propertyname}{conditionally benign}

\newcommand{\propertyNAME}{Conditionally Benign}

\newcommand{\causalgraphname}{causal bandit graph}

\newcommand{\stochgamename}{stochastic bandit problem with \postcontextnames}

\newcommand{\Corral}{Corral}
\newcommand{\UCB}{UCB}
\newcommand{\UCBfull}{upper confidence bound}

\newcommand{\CUCB}{C-UCB}
\newcommand{\CUCBt}{C-UCB-2}
\newcommand{\CUCBfull}{causal upper confidence bound}

\newcommand{\HCUCB}{HAC-UCB}
\newcommand{\HCUCBfull}{hypothesis-tested adaptive causal upper confidence bound}

\newcommand{\UCBsymsmall}{\text{\tiny\upshape UCB}}
\newcommand{\CUCBsymsmall}{\text{\tiny\upshape C}}
\newcommand{\HCUCBsymsmall}{\text{\tiny\upshape HAC}}

\newcommand{\mabalgospace}{\AA_{\text{\tiny\upshape MAB}}}
\newcommand{\pcalgospace}{\AA_{\text{\tiny\upshape C-MAB}}}
\newcommand{\adapcalgospace}{\AA^{\star}_{\text{\tiny\upshape C-MAB}}}

\newcommand{\probspace}{\mathscr{P}}

\newcommand{\responsespace}{\Yy}
\newcommand{\postcontextspace}{\Zz}

\newcommand{\actionspace}{\Aa}
\newcommand{\actionspacenull}{\Aa_0}

\newcommand{\env}[1]{\nu_{#1}}
\newcommand{\baseenv}[1]{\tilde\nu_{#1}}
\newcommand{\specenv}[3]{\nu^{#3}_{#1}}
\newcommand{\envspace}{\probspace(\postcontextspace\times\responsespace)^{\actionspace}}
\newcommand{\envmargspace}{\Pi_{\actionspace,\postcontextspace}}
\newcommand{\policy}[1]{\pi_{#1}}
\newcommand{\policymarg}[1]{\pi^{\margdist}_{#1}}
\newcommand{\policyspace}{\Pi}

\newcommand{\algo}{\mathfrak{a}}
\newcommand{\adaalgo}{{\mathfrak{a}}^{\star}}

\newcommand{\propertyfunc}{\mathfrak{p}}

\newcommand{\margdist}{q}
\newcommand{\baseconddist}{\tilde p}
\newcommand{\conddist}[1]{p^{#1}}

\newcommand{\EEenv}[1]{\EE_{\env{#1}}}
\newcommand{\PPenv}[1]{\PP_{\env{#1}}}

\newcommand{\PPpriorenv}[1]{\PP_{\priormargenv{#1}}}

\newcommand{\PPestenv}[1]{\PP_{\estmargenv{#1}}}
\newcommand{\EEboth}{\EE_{\env{},\policy{}}}
\newcommand{\EEbothucb}{\EE_{\env{},\UCBsymsmall}}
\newcommand{\EEbothc}{\EE_{\env{},\CUCBsymsmall}}
\newcommand{\EEbothhc}{\EE_{\env{},\HCUCBsymsmall}}
\newcommand{\PPboth}{\PP_{\env{},\policy{}}}

\newcommand{\actionval}{a}
\newcommand{\refactionval}{a_0}
\newcommand{\actionvaldum}{a'}
\newcommand{\postcontextval}{z}

\newcommand{\response}[1]{Y_{#1}}
\newcommand{\responserv}[2]{Y_{#1}(#2)}
\newcommand{\postcontext}[1]{Z_{#1}}
\newcommand{\postcontextrv}[2]{Z_{#1}{(#2)}}

\newcommand{\actionrv}[1]{A_{#1}}
\newcommand{\actionrvucb}[1]{A^{\UCBsymsmall}_{#1}}
\newcommand{\actionrvc}[1]{A^{\CUCBsymsmall}_{#1}}
\newcommand{\actionrvhc}[1]{A^{\HCUCBsymsmall}_{#1}}
\newcommand{\martingalerv}[1]{M_{#1}}

\newcommand{\historyrv}[1]{H_{#1}}

\newcommand{\regretalgo}[1]{R_{\env{},\algo(\actionspace,\postcontextspace,\margdist,#1)}(#1)}
\newcommand{\regretadaalgo}[1]{R_{\env{},\adaalgo(\actionspace,\postcontextspace,\margdist,#1)}(#1)}
\newcommand{\regretalgoadapt}[1]{R_{\env{},\algo(\actionspace,\postcontextspace,\env{}(\postcontext{}),#1)}(#1)}

\newcommand{\regretboth}[1]{R_{\env{},\policy{}}(#1)}
\newcommand{\regretucb}[1]{R_{\env{},\UCBsymsmall}(#1)}
\newcommand{\regretc}[1]{R_{\env{},\CUCBsymsmall}(#1)}
\newcommand{\regrethc}[1]{R_{\env{},\HCUCBsymsmall}(#1)}
\newcommand{\regretcustom}[3]{R_{#1,#2}(#3)}

\newcommand{\borelset}{\Bb}

\newcommand{\dummyvecval}{\mathbf{u}}
\newcommand{\condvec}{\mathbf{Z}}

\newcommand{\condvecval}{\mathbf{z}}

\newcommand{\intervenevec}{\mathbf{A}}
\newcommand{\intervenevecval}{\mathbf{a}}

\newcommand{\nodevec}{\mathbf{V}}
\newcommand{\nodevecval}{\mathbf{v}}
\newcommand{\noderv}[1]{V_{#1}}
\newcommand{\nodeval}[2]{v_{#1}^{#2}}
\newcommand{\nodespace}[1]{\Vv_{#1}}
\newcommand{\numnodevals}[1]{k_{#1}}

\newcommand{\graphdists}{\probspace_{\nodevec}}

\newcommand{\graphset}{\Gg}

\newcommand{\numnodes}{M}

\newcommand{\intervenedist}[1]{p_{#1}}
\newcommand{\intervenedistnull}[1]{p'_{#1}}
\newcommand{\basedist}{p}
\newcommand{\basecontextdist}{q}

\newcommand{\probparents}[2]{\mathrm{Pa}^{#1}_{#2}}
\newcommand{\graphparents}[2]{\mathrm{Pa}^{#1}_{#2}}

\newcommand{\lowerconc}{F}
\newcommand{\lowerprobconc}{G}
\newcommand{\actionconc}[1]{E^{\actionspace}_{#1}}
\newcommand{\contextconc}[1]{E^{\postcontextspace}_{#1}}
\newcommand{\margconc}{E^{\nu}}
\newcommand{\concevent}{E}
\newcommand{\switchtime}{\tau_1^{\HCUCBsymsmall}}
\newcommand{\exploretime}{\tau_0^{\HCUCBsymsmall}}
\newcommand{\highprobparam}[1]{\delta_{#1}}

\newcommand{\doaction}{\mathrm{do}}

\newcommand{\priormargenv}[1]{{\tilde \nu}_{#1}}
\newcommand{\estmargenv}[1]{{\hat \nu}_{#1}}
\newcommand{\priormargscale}{\eps}

\newcommand{\optactionval}{a^*_{\env{}}}

\newcommand{\reward}{Y}

\newcommand{\genreward}{Y^{\circ}}
\newcommand{\gencontextreward}{V^{\circ}}

\newcommand{\estcondmeancontext}[1]{{\hat\mu}^{\postcontextspace}_{#1}}
\newcommand{\estcondmeanaction}[1]{{\hat\mu}^{\actionspace}_{#1}}

\newcommand{\estcondmeanboth}[1]{{\hat\mu}^{\actionspace,\postcontextspace}_{#1}}

\newcommand{\margmean}{p^{\actionspace}}
\newcommand{\condmean}{\mu^{\actionspace,\postcontextspace}}

\newcommand{\numobscontext}[1]{\TT^{\postcontextspace}_{#1}}
\newcommand{\numobsaction}[1]{\TT^{\actionspace}_{#1}}
\newcommand{\numobsboth}[1]{\TT^{\actionspace,\postcontextspace}_{#1}}
\newcommand{\actionspaceidx}[1]{\actionspace_{#1}}
\newcommand{\actionvalidx}{i}
\newcommand{\postcontextspaceidx}[1]{\postcontextspace_{#1}}
\newcommand{\postcontextvalidx}{j}

\newcommand{\ucbaction}[1]{\mathrm{UCB}^{\actionspace}_{#1}}
\newcommand{\ucbdiff}[1]{\mathrm{D}^{\actionspace}_{#1}}
\newcommand{\ucbcontext}[1]{\mathrm{UCB}^{\postcontextspace}_{#1}}
\newcommand{\pseudoucb}[1]{\widetilde{\mathrm{UCB}}_{#1}}

\renewcommand{\complement}{'}

%% file: abstract.tex
Multi-armed bandit problems provide a framework to identify the optimal \genactionname{} 
over a sequence of repeated experiments.
Without additional assumptions, minimax optimal performance (measured by cumulative regret) is well-understood.
With access to additional observed variables that $d$-separate the \genactionname{} from the outcome (i.e., they are a \emph{\unobservedname{}}), recent ``causal bandit'' algorithms provably incur less regret.
However, in practice it is desirable to be agnostic to whether observed variables are a \unobservedname{}. 
Ideally, an algorithm should be \emph{adaptive}; that is, perform nearly as well as an algorithm with oracle knowledge of the presence or absence of a \unobservedname{}.
In this work, we formalize and study this notion of adaptivity, and provide a novel algorithm that \emph{simultaneously} achieves (a) optimal regret when a \unobservedname{} is observed, improving on classical minimax algorithms, and (b) significantly smaller regret than recent causal bandit algorithms when the observed variables are not a \unobservedname{}.
Crucially, our algorithm does not require any oracle knowledge of whether a \unobservedname{} is observed.
We also generalize this adaptivity to other conditions, such as the front-door criterion.

%% file: section-files/main-sections/introduction.tex
\section{Introduction}
\label{sec:introduction}

Given a set of \genactionnames{} (\actionnames{}) for a specific experiment, we are interested in learning the best one with respect to some outcome of interest.
Without knowledge of specific causal structure relating the observed variables,
this task is impossible from solely observational data 
\noticml{(c.f.\ Theorem~4.3.2 of \citep{pearl09causality2}).}
\icml{\citep[c.f.\ Theorem~4.3.2 of][]{pearl09causality2}.}
Instead, we seek
the most efficient way to sequentially choose \genactionnames{} for \iid{} repetitions of the experiment, where the main challenge is that we cannot observe the counterfactual effect of \genactionnames{} we did not choose.
Without any structural assumptions beyond \iid{}, one can always learn the best \genactionname{} with high confidence by performing each \genactionname{} a sufficient number of times \citep{bubeck09bestarm}.
In the presence of additional structure---such as a causal graph on observables---this strategy may result in performing suboptimal interventions unnecessarily often.
However, the presence of such structure is often unverifiable, and incorrectly supposing that it exists may catastrophically mislead the \statname{}.
Thus, a fundamental question arises: Can we avoid strong, unverifiable assumptions while simultaneously performing fewer harmful \genactionnames{} when advantageous structure exists?

A natural framework in which to study this question is that of (multi-armed) bandit problems: 
Over a sequence of interactions with the environment, the \statname{} chooses an \actionname{} using their experience of the previous interactions, and then observes the \emph{reward} of the chosen \actionname{}. 
The goal is to achieve comparable performance with what would have been achieved if the \statname{} had chosen the (unknown) optimal \actionname{} in each interaction.
Formally, performance is measured by \emph{regret}, which is the difference of the \emph{cumulative reward} incurred by the \statname{} compared to the optimal \actionname{}.
In this partial-information setting,
\regretname{} induces the classical trade-off between \emph{exploration} (choosing potentially suboptimal \actionnames{} to learn if they're optimal) and \emph{exploitation} (choosing the \actionname{} that empirically appears the best).
In contrast, other measures of performance (e.g., only identifying the average treatment effect or the best \actionname{} at the end of all interactions) do not penalize the \statname{} for performing suboptimal \actionnames{} during exploration, and consequently are insufficient to study our question of interest. 

For an \emph{action set} $\actionspace$ and a \emph{time horizon} 
$T$, 
the minimax optimal \regretname{} for bandits without any assumptions on the data (worst-case) is 
$\tilde\Oo(\sqrt{\abs{\actionspace} T})$ \citep{auer02nonstochastic}, and is achieved by many \algonames{} \citep{bandit20book}.
Recently, \citet{lu20causal} showed that, under additional causal structure, 
a new \algoname{} (\CUCB{}) can achieve improved regret.
In particular, if the \statname{} has access to a variable $\postcontext{}$---taking values in a finite set $\postcontextspace$---that 
\emph{$d$-separates} \citep{pearl09causality2} the \genactionname{} and the reward, 
as well as the 
interventional distribution
of $\postcontext{}$ for each $\actionval\in\actionspace$,
\CUCB{} achieves $\tilde\Oo(\sqrt{\abs{\postcontextspace} T})$ regret.
However, as we show, the performance of \CUCB{} when the $d$-separation assumption fails is orders of magnitude worse than that of \UCB{}.
Is \emph{strict adaptation} possible?
That is, is there an \algoname{} that recovers the guarantee of \CUCB{} when $\postcontext{}$ is a \unobservedname{} and the guarantee of \UCB{} in all other \envnames{}, without advance knowledge of whether $\postcontext{}$ is a \unobservedname{}?

As of yet, there is no general theory of adaptivity in the bandit setting.
The closest we have to a general method is
 the \Corral{} algorithm and its offspring \citep{agarwal17corral,arora21corral}. 
 \Corral{} uses online mirror descent to combine ``base'' bandit algorithms,
 but requires 
that each of the base algorithms is ``stable'' when operating on importance-weighted observations.
Unfortunately, while \UCB{} is stable, simulations reveal this not to be the case for \CUCB{}.
This presents a barrier to adapting to causal structure via \Corral{}-like techniques,
and raises the question of whether there is a new way to achieve 
adaptivity.

\textbf{Contributions.}
We introduce the \emph{\propertyname{}} property for bandit environments: informally, there exists a random variable $\postcontext{}$ such that the conditional distribution of the reward given $\postcontext{}$ is the same for each action $\actionval\in\actionspace$.
We show that the \propertyname{} property is (a) strictly weaker than the assumption of \citet{lu20causal} (in their proofs, they actually assume that all causal parents of the reward are observed); (b) equivalent to $\postcontext{}$ being a \unobservedname{} when $\actionspace$ is all \genactionnames{}; and (c) implied by the front-door criterion \citep{pearl09causality2} when $\actionspace$ is all \genactionnames{} except the \emph{null \genactionname{}} (i.e., a pure observation).
We then prove that any \algoname{} that achieves optimal worst-case regret must incur suboptimal regret for some \propertyname{} \envname{}, and hence strict adaptation to the \propertyname{} property  is impossible.
Despite this, 
we introduce the \HCUCBfull{} \algoname{} (\HCUCB{}), which provably
(a) achieves non-vacuous (sublinear in $T$) \regretname{} at all times without any assumptions,
(b) recovers the improved performance of \CUCB{} for \propertyname{} \envnames{},
and (c) performs as well as \UCB{} in certain \envnames{} where \CUCB{} 
and related \algonames{}
\noticml{(such as those studied in \citep{lu21unknown,nair21budget})}
\icml{\citep[such as those studied in][]{lu21unknown,nair21budget}} 
incur linear \regretname{}.
Empirically, we observe these performance improvements 
on simulated data.

\textbf{Impact.}
Recently, multiple works have developed causal bandit \algonames{} that achieve improved performance in the presence of advantageous causal relationships 
\noticml{(initiated by \citet{bareinboim15causal} and \citet{lattimore16causal}; see \cref{sec:literature} for more literature).}
\icml{\citep[initiated by][see \cref{sec:literature} for more literature]{bareinboim15causal,lattimore16causal}.} 
Further, the last decade has seen a flurry of work in bandits on designing \algonames{} that recover worst-case \regretname{} bounds while simultaneously performing significantly better in advantageous settings, without requiring advance knowledge of which case holds \citep[e.g.,][]{bubeck12bandits,seldin14sao,thune18easy,lykouris18corruption,abbasiyadkori18bobwarm}. 
However, to the best of our knowledge, no existing work studies \algonames{} that achieve \emph{adaptive regret guarantees with respect to causal structure}.
The present work provides a framework that expands the study of adaptive decision making to the rich domain of causal inference.

%% file: section-files/main-sections/bandit-defs.tex
\section{Preliminaries}\label{sec:bandit-defs}

\subsection{Problem Setting}

We consider a general extension of the usual  bandit setting where, in addition to a reward corresponding to the \actionname{} played, the \statname{} observes some additional variables \emph{after} choosing their \actionname{}; we call this the \emph{\postcontextname{}}.
This is distinct from the contextual bandit problem, where the \statname{} has access to side-information \emph{before} choosing their \actionname{}.
\preprint{The causal interpretation of our setting is that the \actionname{} chosen by the \statname{} not only impacts the reward, but also other aspects of the \envname{} that may be used to more efficiently share information between \actionnames{}.}

Let $\responsespace = [0,1]$ be the \emph{reward space}\footnote{Our results hold for $\responsespace=\Reals$ using sub-Gaussian rewards with bounded mean at the expense of constants.}, $\postcontextspace$ be a finite set of values for the \postcontextname{} to take,
and $\probspace(\postcontextspace \times \responsespace)$ denote the set of joint probability distributions. 
For any $\basedist \in \probspace(\postcontextspace \times \responsespace)$ and $(\postcontextspace, \responsespace)$-valued random variable $(\postcontext{},\response{}\,)$, let $\basedist(\postcontext{})$ and $\basedist(\response{}\setdelim \postcontext{})$ denote the marginal and conditional distributions respectively.
Let $\EE_\basedist$ and $\PP_\basedist$ denote expectation and probability operators
under $\basedist$.

The \emph{\stochgamename{}} proceeds as follows: For each round $t\in[T]$, the \statname{} selects $\actionrv{t}\in\actionspace$ 
while simultaneously $\{(\postcontextrv{t}{\actionval},\responserv{t}{\actionval}): \actionval\in\actionspace\}$ are independently sampled from the \emph{\envname{}}, which is any family of distributions $\env{} = \{\env{\actionval}: \actionval\in\actionspace\} \in \envspace$ indexed by the \actionname{} set. 
The \statname{} only observes $(\postcontextrv{t}{\actionrv{t}},\responserv{t}{\actionrv{t}})$ and receives reward $\responserv{t}{\actionrv{t}}$.
From a causal perspective,
$(\postcontextrv{t}{\actionval},\responserv{t}{\actionval})_{\actionval\in\actionspace}$ corresponds to the potential outcome vector, and under causal consistency, $(\postcontextrv{t}{\actionrv{t}},\responserv{t}{\actionrv{t}})$ corresponds to the observed data $(\postcontext{t},\response{t})$ under the chosen intervention $\actionrv{t}$.

The \emph{\historyname{} up to round $t$} is the random variable $\historyrv{t} = (\actionrv{s}, \postcontextrv{s}{\actionrv{s}}, \responserv{s}{\actionrv{s}})_{s\in[t]}$.
\noticml{
A \emph{\policyname{}} is a sequence of measurable maps from the \historyname{} to the \actionname{} set, denoted 
\*[
	\policy{} = (\policy{t})_{t\in[T]}
	\in 
	\policyspace(\actionspace, \postcontextspace, T)
	\defas
	\prod_{t=1}^T \Big\{(\actionspace\times \postcontextspace\times \responsespace)^{t-1} \to \actionspace\Big\}.
\]
}
\icml{
A \emph{\policyname{}} is a sequence of measurable maps from the \historyname{} to the \actionname{} set, denoted by $\policy{} = (\policy{t})_{t\in[T]} \in \policyspace(\actionspace, \postcontextspace, T)$, where
\*[
	\policyspace(\actionspace, \postcontextspace, T)
	\defas
	\prod_{t=1}^T \Big\{(\actionspace\times \postcontextspace\times \responsespace)^{t-1} \to \actionspace\Big\}.
\]
}
The \statname{} chooses a policy in advance to select their \actionname{} on each round according to $\actionrv{t} = \policy{t}(\historyrv{t-1})$. 
Clearly, an \envname{} $\env{}$ and a \policyname{} $\policy{}$ together define a joint probability distribution on $(\actionrv{t}, \postcontextrv{t}{\actionval}, \responserv{t}{\actionval})_{t\in[T], \actionval\in\actionspace}$ (which includes the ``counterfactuals'' not seen by the \playername{}).  Let $\EEboth$ denote expectation under this joint distribution. 
The performance of a \policyname{} under an \envname{} is quantified by the \emph{\regretname{}}
\*[
	\regretboth{T}
	= T \cdot \max_{\actionval\in\actionspace}\EEenv{\actionval}[\response{} \,]
	- \EEboth \sum_{t=1}^T \EEenv{\actionrv{t}}[\response{} \,].
\]

\subsection{Specific Algorithms}\label{sec:algos}

Classical  bandit algorithms take $\abs{\actionspace}$ and $T$ as inputs.
The dependence on $T$ can often be dropped using the doubling trick or more sophisticated techniques such as decreasing learning rates, but we do not focus on these refinements in this work, instead allowing $T$ as an input. 
In order to account for the additional information provided by the \postcontextname{}, 
we also consider algorithms that take $\abs{\postcontextspace}$ as an input. 
By restricting dependence to only the cardinality of $\actionspace$ and $\postcontextspace$, we explicitly suppose that there is no additional structure to exploit on these spaces; much work in the bandit literature has focused on such structure through linear or Lipschitz rewards, but we defer these extensions to future work in favour of focusing on adaptivity.
For notational simplicity, we denote algorithmic dependence on $\actionspace$ or $\postcontextspace$ even though the dependence is actually through their cardinality (i.e., the labellings of items in the sets are arbitrary).

In the causal bandit literature \citep{lattimore16causal,lu20causal,nair21budget}, it is common to suppose that the algorithm also receives distributional information relating actions to intermediate variables. 
In particular, if the (unknown) \envname{} is $\env{}$, prior work supposes that the algorithm has access to the (interventional) marginal distributions $\env{}(\postcontext{}) = \{\env{\actionval}(\postcontext{}): \actionval\in\actionspace\}$. 
In this work, we suppose instead that the algorithm has access to a collection of \emph{approximate} marginal distributions $\priormargenv{}(\postcontext{}) = \{\priormargenv{\actionval}(\postcontext{}): \actionval\in\actionspace \}$; for example, these could be an estimate of $\env{}(\postcontext{})$ that was learned offline.
Ideally, $\priormargenv{}(\postcontext{})$ will be close to $\env{}(\postcontext{})$, but our novel method is entirely adaptive to this assumption: regardless of how well $\priormargenv{}(\postcontext{})$ approximates $\env{}(\postcontext{})$, \HCUCB{} incurs sublinear regret.

We now introduce additional notation to define the algorithms of interest in this work.
Suppose that $\actionspace$, $\postcontextspace$, $\priormargenv{}(\postcontext{})$, and $T$ are all fixed in advance, as well as a confidence parameter $\highprobparam{} = \highprobparam{T} \in(0,1)$. 
For each $t \in [T]$, $\postcontextval\in\postcontextspace$, and $\actionval\in\actionspace$, define the number of the first $t$ rounds on which $\postcontextval$ was observed by $\numobscontext{t}(\postcontextval) = 1 \vee \sum_{s=1}^t \ind\{\postcontextrv{s}{\actionrv{s}}=\postcontextval\}$ (where $a \vee b$ = $\max\{a,b\}$), and similarly the number of rounds on which $\actionval$ was chosen by $\numobsaction{t}(\actionval) = 1 \vee \sum_{s=1}^t \ind\{\actionrv{s}=\actionval\}$. 
Further, define 
the empirical mean estimate for the reward under the distribution induced by choosing the \actionname{} $\actionval$ as
$\estcondmeanaction{t}(\actionval) = [\numobsaction{t}(\actionval)]^{-1} \sum_{s=1}^t \responserv{s}{\actionrv{s}} \ind\{\actionrv{s}=\actionval\}$
and
the empirical conditional mean estimate for the reward given that $\postcontextval$ was observed as
$\estcondmeancontext{t}(\postcontextval) = [\numobscontext{t}(\postcontextval)]^{-1} \sum_{s=1}^t \responserv{s}{\actionrv{s}} \ind\{\postcontextrv{s}{\actionrv{s}}=\postcontextval\}$.
Define $\ucbaction{t}(\actionval) = \estcondmeanaction{t}(\actionval) + \sqrt{\log (2/\highprobparam{}) / (2\numobsaction{t}(\actionval))}$, $\ucbcontext{t}(\postcontextval) = \estcondmeancontext{t}(\postcontextval) + \sqrt{\log (2/\highprobparam{}) / (2\numobscontext{t}(\postcontextval))}$, and $\pseudoucb{t}(\actionval) = \sum_{\postcontextval\in\postcontextspace} \ucbcontext{t}(\postcontextval) \PPpriorenv{\actionval}[\postcontext{}=\postcontextval]$.

Using these objects, we define three algorithms, each of which produces actions that are $\historyrv{t}$-measurable. 
The \UCBfull{} algorithm 
\noticml{(\UCB{}, \citep{auer02nonstochastic})}
\icml{\citep[\UCB{},][]{auer02nonstochastic}}
is defined by $\actionrvucb{t+1} = \argmax_{\actionval\in\actionspace} \ucbaction{t}(\actionval)$, and the \CUCBfull{} algorithm 
\noticml{(\CUCB{}, \citep{lu20causal})} 
\icml{\citep[\CUCB{},][]{lu20causal}}
is defined by $\actionrvc{t+1} = \argmax_{\actionval\in\actionspace} \pseudoucb{t}(\actionval)$, where ties are broken by using some predetermined ordering on $\actionspace$.
Finally, we define a new combination of these two methods, which we call the \HCUCBfull{} algorithm (\HCUCB) and describe precisely in \cref{alg:hyp-test}; we denote its actions by $\actionrvhc{t+1}$.

\SetKwIF{If}{ElseIf}{Else}{if}{}{else if}{else}{end if}%
\SetKwFor{While}{while}{}{end while}%
\SetKwRepeat{Do}{do}{while}

\begin{algorithm}[h]
\SetAlgoLined
\caption{\HCUCB($\actionspace$, $\postcontextspace$, $T$, $\priormargenv{}(\postcontext{})$)}\label{alg:hyp-test}
\textbf{do} Play each $\actionval\in\actionspace$ for $\lceil 4 \sqrt{T} \,/\, \abs{\actionspace} \, \rceil$ rounds, and let $\estmargenv{\actionval}(\postcontext{})$ be the MLE of $\env{\actionval}(\postcontext{})$\\
	\If{$\sup_{\actionval\in\actionspace} \sum_{\postcontextval\in\postcontextspace} \abs[2]{\PPpriorenv{\actionval}[\postcontext{} = \postcontextval] - \PPestenv{\actionval}[\postcontext{} = \postcontextval]} > 2 T^{-1/4}\sqrt{\abs{\actionspace}\abs{\postcontextspace}\log T}$}{
		\textbf{replace} $\priormargenv{}(\postcontext{}) \longleftarrow \estmargenv{}(\postcontext{})$
	}
\textbf{do} Play each $\actionval\in\actionspace$ for $\lceil \sqrt{T} \,/\, \abs{\actionspace} \, \rceil$ rounds\\
\textbf{set} flag = True\\
\While{$t \leq T$}{
	\eIf{$\mathrm{flag}$}{
		\tcc{Check if either of the two conditions fail}
		\textbf{set} $\ucbdiff{t-1}(\actionval)  = \ucbaction{t-1}(\actionval) - \pseudoucb{t-1}(\actionval) + \frac{\sqrt{\abs{\actionspace}\abs{\postcontextspace}\log T}}{T^{1/4}}$\\
		\For{$\actionval\in\actionspace$}{
			\If{\textbf{\upshape not} $-2\sum\limits_{\postcontextval\in\postcontextspace} \sqrt{\tfrac{\log T}{\numobscontext{t-1}(\postcontextval)}}\, \PPpriorenv{\actionval}[\postcontext{}{}=\postcontextval] \leq \ucbdiff{t-1}(\actionval) \leq 2\sqrt{\tfrac{\log T}{\numobsaction{t-1}(\actionval)}} + 2\frac{\sqrt{\abs{\actionspace}\abs{\postcontextspace}\log T}}{T^{1/4}}$}{
				\textbf{set} flag = False;
				\textbf{break}
			}
		}
		\tcc{If conditions pass, play \CUCB{}}
		\eIf{$\mathrm{flag}$}{
			\textbf{set} $\actionrvhc{t} = \actionrvc{t}$;
		}{
			\textbf{set} $\actionrvhc{t} = \actionrvucb{t}$;
		}
	}{
		\tcc{If conditions ever fail, play \UCB{} forever}
		\textbf{set} $\actionrvhc{t} = \actionrvucb{t}$;
	}
}
\end{algorithm}
\newcommand{\True}{\textbf{true}}

Heuristically, \HCUCB{} has an initial exploration period to ensure that $\priormargenv{}(\postcontext{})$ is sufficiently accurate---if not, it is replaced with the maximum likelihood estimate (MLE) of the marginals---and then optimistically plays \CUCB{} until there is sufficient evidence that the \envname{} is not \propertyname{}.
The switch from \CUCB{} to \UCB{} is decided by a hypothesis test performed on each round, which uses the confidence intervals that will hold if the \envname{} is \propertyname{}.
When the arm mean estimates of \UCB{} and \CUCB{} disagree, this provides evidence that the \envname{} is not \propertyname{}, and the evidence is considered sufficient to switch when the size of the disagreement is large compared to the size of the confidence intervals themselves.
As we illustrate in the proof of the regret bounds, with high probability this test will not induce a switch for a \propertyname{} \envname{}, and will sufficiently limit the regret incurred by \CUCB{} if the \envname{} is not \propertyname{}.

%% file: section-files/main-sections/causal-theorems.tex
\section{\propertyNAME{} Property}\label{sec:causal-theorems}

We now formalize the main property that \HCUCB{} will adaptively exploit.

\begin{definition}\label{def:main-property}
An \envname{} \emph{$\env{}\in\envspace$ is \propertyname{}} if and only if there exists $\basedist\in\probspace(\postcontextspace\times\responsespace)$ such that for each $\actionval\in\actionspace$, $\env{\actionval}(\postcontext{}) \ll \basedist(\postcontext{})$ and
$\env{\actionval}(\response{} \setdelim \postcontext{}) = \basedist(\response{} \setdelim \postcontext{})$ $\basedist$-a.s.
\end{definition}

This definition does not require any causal terminology to define or use for \regretname{} bounds, but we now instantiate it for the causal setting.
For a collection of finite random variables $\nodevec$ and a (potentially) continuous random variable $\reward{}$, let $\graphdists$ be the set of all joint probability distributions with strictly positive marginal probabilities on $\nodevec$.
Fix a DAG $\graphset$ on $(\nodevec, \reward{})$ such that $\reward{}$ is a leaf and two disjoint sets $\condvec \subseteq \nodevec$ and $\intervenevec \subseteq \nodevec$ such that $\graphparents{\graphset}{\intervenevec} \subseteq \nodevec\setminus\condvec$.
Let $\actionspace$ be the set of all possible $\doaction$ interventions on $\intervenevec$, and for each $\actionval\in\actionspace$ let $\intervenedist{\actionval}$ denote the interventional distribution (\cref{def:causal-intervention}).
This structure suggests a graphical analogue of the \propertyname{} property.

\begin{definition}
For any DAG $\graphset$ and 
$\actionspace'\subseteq\actionspace$, \emph{$(\graphset, \actionspace')$ is \propertyname{}} if and only if for all $\basedist\in\graphdists$ that are Markov relative to $\graphset$, $\{\intervenedist{\actionval}(\condvec, \response{}\,): \actionval \in \actionspace'\}$ is \propertyname{}.
\end{definition}

We now connect the \propertyname{} property to $d$-separation (\cref{def:d-sep}) and the front-door criterion (\cref{def:front-door}).
All proofs are deferred to \cref{sec:causal-proofs}, along with standard notation and definitions from the causal literature.

\begin{theorem}\label{fact:benign-equals-dsep}
$\condvec$ $d$-separates $\response{}$ from $\intervenevec$ on $\graphset$ if and only if $(\graphset, \actionspace)$ is \propertyname{}.
\end{theorem}

This equivalence is a strict specialization of the \propertyname{} property to the causal setting. In particular, to define \propertyname{}, we need not require all possible \genactionnames{} be allowed.
 Define $\actionspacenull$ to be $\actionspace$ with the \nullaction{} (observational) \genactionname{} removed, and let $\graphset_{\bar\intervenevec}$ denote $\graphset$ with all edges directed into $\intervenevec$ removed.

\begin{theorem}\label{fact:benign-equals-dsep-null}
$\condvec$ $d$-separates $\response{}$ from $\intervenevec$ on $\graphset_{\bar\intervenevec}$ if and only if 
$(\graphset, \actionspacenull)$ is \propertyname{}.
\end{theorem}

The benefits of discarding the \nullaction{} \genactionname{} are demonstrated by the following fact.

\begin{lemma}\label{fact:front-implies-dsep}
$\condvec$ $d$-separates $\response{}$ from $\intervenevec$ on $\graphset_{\bar\intervenevec}$
if $\condvec$ satisfies the front-door criterion for $(\intervenevec, \response{})$ on $\graphset$.
\end{lemma}

We visualize
the preceding results in \cref{fig:causal}. 
In graph (a), $\postcontext{}$ $d$-separates the \genactionname{} from the reward, and hence any Markov relative distribution (\cref{def:markov-relative}) will induce a \propertyname{} \envname{}.
Graph (b) corresponds to a setting where one cannot hope to always improve performance due to the direct effect of the \genactionname{} on the reward, and consequently the \envname{} need not be \propertyname{}. 
In graph (c), the presence of the unobserved confounder $U$ means that $\postcontext{}$ does not $d$-separate the \genactionname{} from the reward.
However, if the \nullaction{} \genactionname{} is not considered, the arrow from $U$ to $A$ is never applicable, and hence any Markov relative distribution on the modified DAG will induce a \propertyname{} \envname{}. Specifically, graph (c) satisfies the front-door criterion, revealing that the \propertyname{} property captures that this setting is still benign for decision-making, even though the conditions assumed by \citet{lu20causal} do not hold. 
Finally, in graph (d), the unobserved confounder $U$ once again violates $d$-separation, but also the front-door criterion is not satisfied because of the back-door path from $\postcontext{}$ to $\reward{}$.
Hence, even discarding the \nullaction{} \genactionname{} does not guarantee that the \envname{} will be \propertyname{}. 

\begin{figure}[h]
\centering
\begin{tikzpicture}
    \node (a) at (0,0) {(a)};
    \node[state] (x1) at (0,-0.75) {$A$};
    \node[state] (z1) [right =of x1, xshift=-0.5cm] {$Z$};
    \node[state] (y1) [right =of z1, xshift=-0.5cm] {$Y$};

    \path (x1) edge (z1);
    \path (z1) edge (y1);

    \node (b) at (4,0) {(b)};
    \node[state] (x2) at (4,-0.75) {$A$};
    \node[state] (z2) [right =of x2, xshift=-0.5cm] {$Z$};
    \node[state] (y2) [right =of z2, xshift=-0.5cm] {$Y$};

    \path (x2) edge (z2);
    \path (z2) edge (y2);
    \path (x2) edge[bend left=40] (y2);

    \node (c) at (0,-2) {(c)};
    \node[state] (x3) at (0,-2.75) {$A$};
    \node[state] (z3) [right =of x3, xshift=-0.5cm] {$Z$};
    \node[state] (y3) [right =of z3, xshift=-0.5cm] {$Y$};
    \node[state, dashed] (u3) [above =of z3, yshift=-0.7cm] {$U$};

    \path (x3) edge (z3);
    \path (z3) edge (y3);
    \path (u3) edge (x3);
    \path (u3) edge (y3);

    \node (d) at (4,-2) {(d)};
    \node[state] (x4) at (4,-2.75) {$A$};
    \node[state] (z4) [right =of x4, xshift=-0.5cm] {$Z$};
    \node[state] (y4) [right =of z4, xshift=-0.5cm] {$Y$};
    \node[state, dashed] (u4) [above =of z4, yshift=-0.7cm] {$U$};

    \path (x4) edge (z4);
    \path (z4) edge (y4);
    \path (u4) edge (z4);
    \path (u4) edge (y4);
\end{tikzpicture}
\caption{DAGs to illustrate the \propertyname{} property. $A$ is the \genactionname{}, $\postcontext{}$ is the \postcontextname{}, $Y$ is the reward, and $U$ is an unobserved variable. 
$(\graphset, \actionspace)$ is \propertyname{} for (a) but only $(\graphset, \actionspacenull)$ is \propertyname{} for (c).
For (b) and (d) the \envname{} need not be \propertyname{}.}
\label{fig:causal}
\end{figure}

%% file: section-files/main-sections/main-theorems.tex
\section{Analysis of Bandit Algorithms}\label{sec:main-theorems}

We now study the impact of the \propertyname{} property 
on \regretname{}.
All proofs are deferred to \cref{sec:main-proofs}.
First, recall the standard \regretname{} bound for \UCB{}, with constants tuned to rewards in $[0,1]$.

\begin{theorem}[Theorem~7.2 of \noticml{\citep{bandit20book}}\icml{\citealt{bandit20book}}]\label{fact:existing-ucb}
For all $\actionspace$, $\postcontextspace$, $T$, and $\env{}\in\envspace$, 
if $\highprobparam{}=2/T^2$
\*[
	\regretucb{T}
	\leq 2\abs{\actionspace} + 4\sqrt{2\abs{\actionspace} T \log T}.
\]
\end{theorem}

Second, we generalize the main result of \citet{lu20causal} by relaxing two assumptions: we only require that the \envname{} is \propertyname{}, and we allow for approximate marginal distributions using the following definition. Later, this will enable us to trade-off approximation error of $\priormargenv{}(\postcontext{})$ with online estimation of $\env{}(\postcontext{})$ in order to ensure that \HCUCB{} \emph{always} incurs sublinear regret.

\begin{definition}
For any $\priormargscale\geq0$, $\priormargenv{}(\postcontext{})$ and $\env{}(\postcontext{})$ are \emph{$\priormargscale$-close} if 
\*[
	\sup_{\actionval\in\actionspace} \sum_{\postcontextval\in\postcontextspace} \abs[2]{\PPpriorenv{\actionval}[\postcontext{} = \postcontextval] - \PPenv{\actionval}[\postcontext{} = \postcontextval]} \leq \priormargscale.
\]
\end{definition}

\begin{theorem}[Refined Theorem~1 of \noticml{\citep{lu20causal}}\icml{\citealt{lu20causal}}]\label{fact:existing-causal}
For all $\eps>0$, $\actionspace$, $\postcontextspace$, $T$, and \propertyname{} $\env{}\in\envspace$,
if $\priormargenv{}(\postcontext{})$ and $\env{}(\postcontext{})$ are $\priormargscale$-close and $\highprobparam{}=2/T^2$ then
\*[
	\regretc{T}
	&\leq 
	2\abs{\postcontextspace}
	+ 6\sqrt{\abs{\postcontextspace}T\log T}
	+ (\log T)\sqrt{2T} 
	\icml{\\ &\qquad} 
	+ 2\priormargscale(1+\sqrt{\log T})T.
\]
\end{theorem}
\begin{remark}
\citet{lu20causal} assume that $\priormargscale=0$. Our result implies that 
an approximation error of $\priormargscale = \sqrt{\abs{\postcontextspace}/T}$ is sufficient to achieve the optimal rate.
\end{remark}

Next, we motivate our introduction of a new \algoname{} by showing that \CUCB{} can catastrophically fail when the \envname{} is \emph{not} \propertyname{}, incurring \regretname{} that is linear in $T$ (which is as bad as possible for bounded rewards).

\begin{theorem}\label{fact:causal-failure}
For every $\actionspace$ and $\postcontextspace$ with $\abs{\actionspace} \geq 2$, there exists $\env{}\in\envspace$ such that even if $\priormargenv{}(\postcontext{}) = \env{}(\postcontext{})$,
for all possible settings of the confidence parameters, $\highprobparam{T}$,
\*[
	\lim_{T \to \infty} \frac{\regretc{T}}{T} \geq 1/120.
\]
\end{theorem}	

\begin{remark}
This lower bound is specifically for \CUCB{}. However, any \algoname{} that relies on eliminating \actionnames{} from consideration via assumption rather than data is susceptible to such an issue.
In particular, this result is easily modified to apply for C-TS \citep{lu20causal} and \CUCBt{} \citep{nair21budget}.
Other causal \algonames{}
\noticml{(e.g., Parallel Bandit \citep{lattimore16causal})}
\icml{\citep[e.g., Parallel Bandit; see][]{lattimore16causal}} 
also intuitively suffer from the issue that our construction exploits, although a different argument must be made since these rely on different causal structure.	
\end{remark}

We now state our \regretname{} upper bound for our new \algoname{}, \HCUCB{}. In \cref{fact:adaptive-impossible-short}, we will show it is \emph{impossible} to always achieve the optimal \regretname{} without knowledge of whether a \unobservedname{} is observed, but the following theorem 
shows \emph{some} adaptivity is always possible.
Crucially, \HCUCB{} achieves sublinear regret without any assumptions on $\env{}$ or $\priormargenv{}(\postcontext{})$.
For a more detailed breakdown of the constants, see \cref{eqn:hyptest-constants}.

\begin{theorem}[Main Result]\label{fact:improved-causal}
For all $\actionspace$, $\postcontextspace$, $T \geq 25\abs{\actionspace}^2$, $\env{}\in\envspace$, and $\priormargenv{}(\postcontext{})\in\probspace(\postcontextspace)^\actionspace$,
\*[
	\regrethc{T}
	&\leq 4\abs{\actionspace} 
	+ 11 \, T^{3/4} (\log T) \sqrt{\abs{\actionspace}\abs{\postcontextspace}}
	+ 15 \sqrt{(\abs{\actionspace}+\abs{\postcontextspace})T\log T}
	+ 5(\log T)\sqrt{T}.
\]
For all $\priormargscale \leq T^{-1/4} \sqrt{\abs{\actionspace}\abs{\postcontextspace}\log T}$, if $\env{}$ is \propertyname{} and $\priormargenv{}(\postcontext{})$ and $\env{}(\postcontext{})$ are $\priormargscale$-close then
\*[
	\regrethc{T}
	\leq 4\abs{\actionspace}+2\abs{\postcontextspace}
	+ 6\sqrt{\abs{\postcontextspace}T\log T}
	+ 4(\log T)\sqrt{T}
	+ 2\priormargscale (1+\sqrt{\log T})T.
\]
\end{theorem}

It is an open problem whether the dependence on $T^{3/4}$ is tight.
In \cref{fact:adaptive-impossible-short}, we will show that it is impossible to obtain worst-case regret of size $\sqrt{\abs{\actionspace} T}$ while still achieving improved \regretname{} on \propertyname{} environments. 
However, it may be possible to improve the dependence on $T$, and the role of logarithmic factors in how much improvement is possible remains to be understood.
Towards improving this result, 
we now show that there exists an \envname{} that forces \CUCB{} to incur linear \regretname{} yet \HCUCB{} will switch to following \UCB{} (and hence incurs $\sqrt{\abs{\actionspace} T\log T}$ regret at worst). That is, \HCUCB{} recovers the improved performance of \CUCB{} when the \propertyname{} property holds, is never worse than \CUCB{}, and optimally outperforms \CUCB{} in some settings.

\begin{theorem}\label{fact:improved-causal-special}
There exists a constant $C$ such that for any $\actionspace$ and $\postcontextspace$ with $\abs{\actionspace} \geq 2$, there exists $\env{}\in\envspace$ so that for any $\highprobparam{T}$ used for \CUCB{} with $\priormargenv{}(\postcontext{})=\env{}(\postcontext{})$,
\*[
	\lim_{T \to \infty} \frac{\regretc{T}}{T} \geq 1/C,
\]
and if $\priormargenv{}(\postcontext{})=\env{}(\postcontext{})$ is used for \HCUCB{} then
\*[
	\lim_{T \to \infty}\frac{\regrethc{T}}{\abs{\actionspace} + \abs{\postcontextspace} + \sqrt{\abs{\actionspace} T \log T} + (\log T)\sqrt{T}} \leq C.
\]
\end{theorem}
\begin{remark}
\cref{fact:improved-causal-special} could be stated with $\priormargenv{}(\postcontext{})$ only $\eps$-close to $\env{}(\postcontext{})$, but for simplicity we have supposed $\eps=0$ to highlight the role of the \propertyname{} property.
\end{remark}

%% file: section-files/main-sections/simulations.tex
\section{Simulations}\label{sec:simulations}

We now study the empirical performance of these algorithms in two key settings, corresponding to a \propertyname{} \envname{} and the lower bound \envname{} from \cref{fact:causal-failure}. 
We compare our algorithm \HCUCB{} with \UCB{} \citep{auer02stochastic}, \CUCB{} \citep{lu20causal}, \CUCBt{} \citep{nair21budget}, and \Corral{} (for which we use the learning rate and algorithm prescribed by \citep{agarwal17corral} with the epoch-based, scale-sensitive \UCB{} prescribed by \citep{arora21corral}); for all algorithms, we use the parameters that are optimal as prescribed by existing theory.
To focus solely on the impact of the \propertyname{} property, we set $\priormargenv{}(\postcontext{}) = \env{}(\postcontext{})$.
The results of this section are a representative simulation demonstrating empirically that
(a) for worst-case \envnames{}, both \CUCB{} and \CUCBt{} incur linear \regretname{}, while \HCUCB{} successfully switches to incur sublinear \regretname{} to compete with \Corral{} and \UCB{}, and (b) for \propertyname{} \envnames{}, \HCUCB{} and \CUCB{} enjoy improved performance compared to \UCB{}, \Corral{}, and \CUCBt{}, all three of which have regret growing like $\sqrt{\abs{\actionspace} T}$.
Implementation details are available in \cref{sec:simulation-details} and
code
can be found at \href{https://github.com/blairbilodeau/adaptive-causal-bandits}{https://github.com/blairbilodeau/adaptive-causal-bandits}.

\subsection{\propertyNAME{} \envName{}}

First, we consider a \propertyname{} \envname{}. 
Taking the gap $\Delta = \sqrt{\abs{\actionspace}(\log T) / T}$, the fixed conditional distribution for $\postcontextspace = \{0,1\}$ is
$\reward{} \setdelim \postcontext{}
    \sim \bernoullidist(1/2 + (1-\postcontext{})\Delta)$.
Then, for a small $\eps$ (we take $\eps=0.0005$), we set $\PP_{\env{1}}[\postcontext{}=0]=1-\eps$ and $\PP_{\env{\actionval}}[\postcontext{}=0]=\eps$ for all other $\actionval\in\actionspace\setminus\{1\}$. Thus, $\optactionval=1$, and the \actionname{}s are separated by $\Delta$. 
In summary, each $\postcontextval\in\postcontextspace$ has positive probability of being observed, yet each \actionname{} nearly deterministically fixes $\postcontext{}$.

In \cref{fig:sims} (left panel) we observe three main effects: (a) \CUCB{} and \HCUCB{} perform similarly (their regret curves overlap), both achieving much smaller \regretname{} that remains unchanged by increasing $\abs{\actionspace}$, (b) \UCB{} grows at the worst-case rate of roughly $\sqrt{\abs{\actionspace} T \log T}$, not taking advantage of the \propertyname{} property, and (c) neither \Corral{} nor \CUCBt{} realize the benefits of the \propertyname{} property, since the \regretname{} increases with $\abs{\actionspace}$ and empirically they perform worse than \UCB{}. We note that the x-axis starts at $T=500$ to satisfy the minor requirement of $T > \abs{\actionspace}^2$.

\begin{figure}[h]
\centering
\neurips{\includegraphics[scale=0.32]{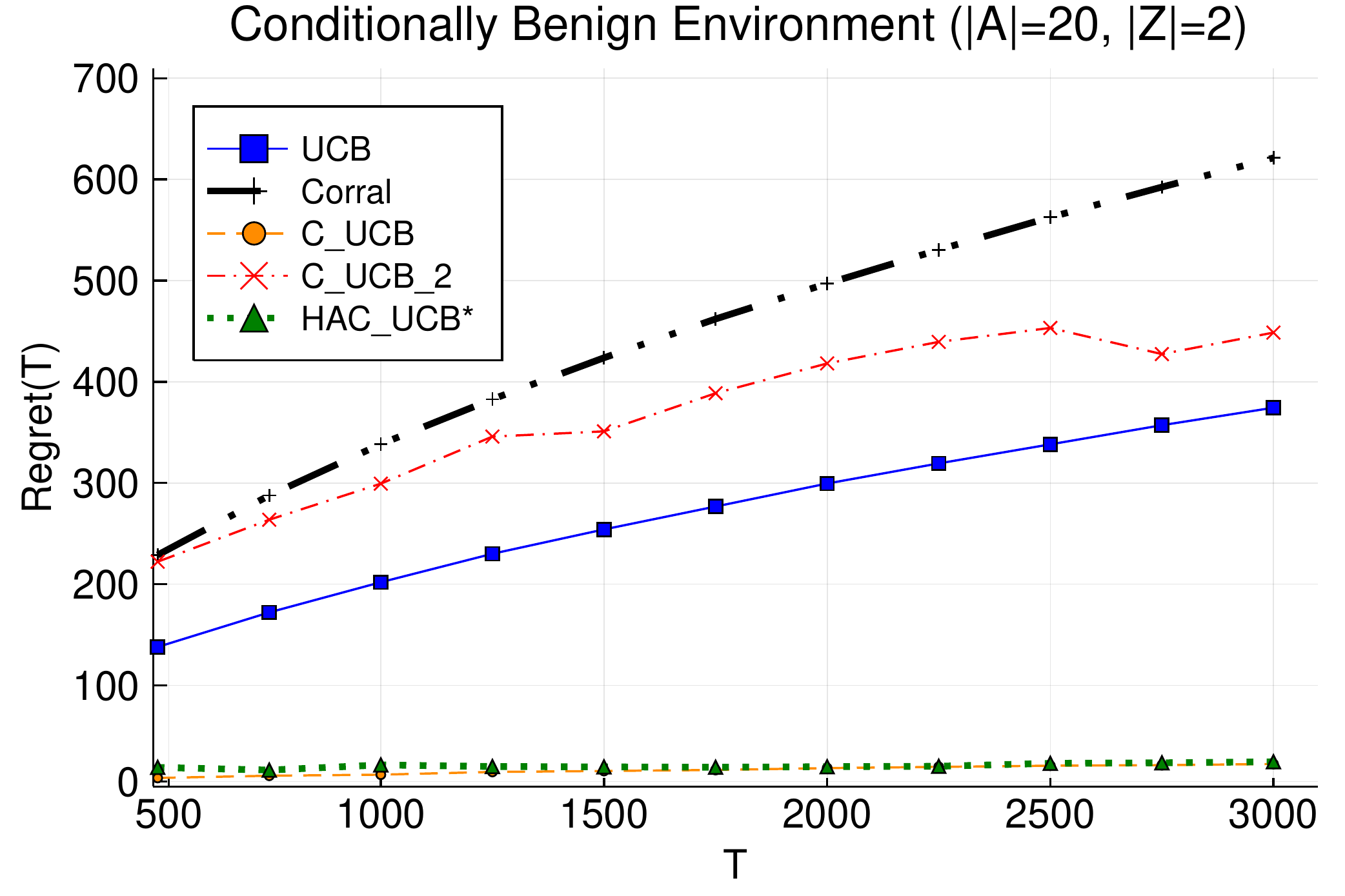}}
\neurips{\includegraphics[scale=0.32]{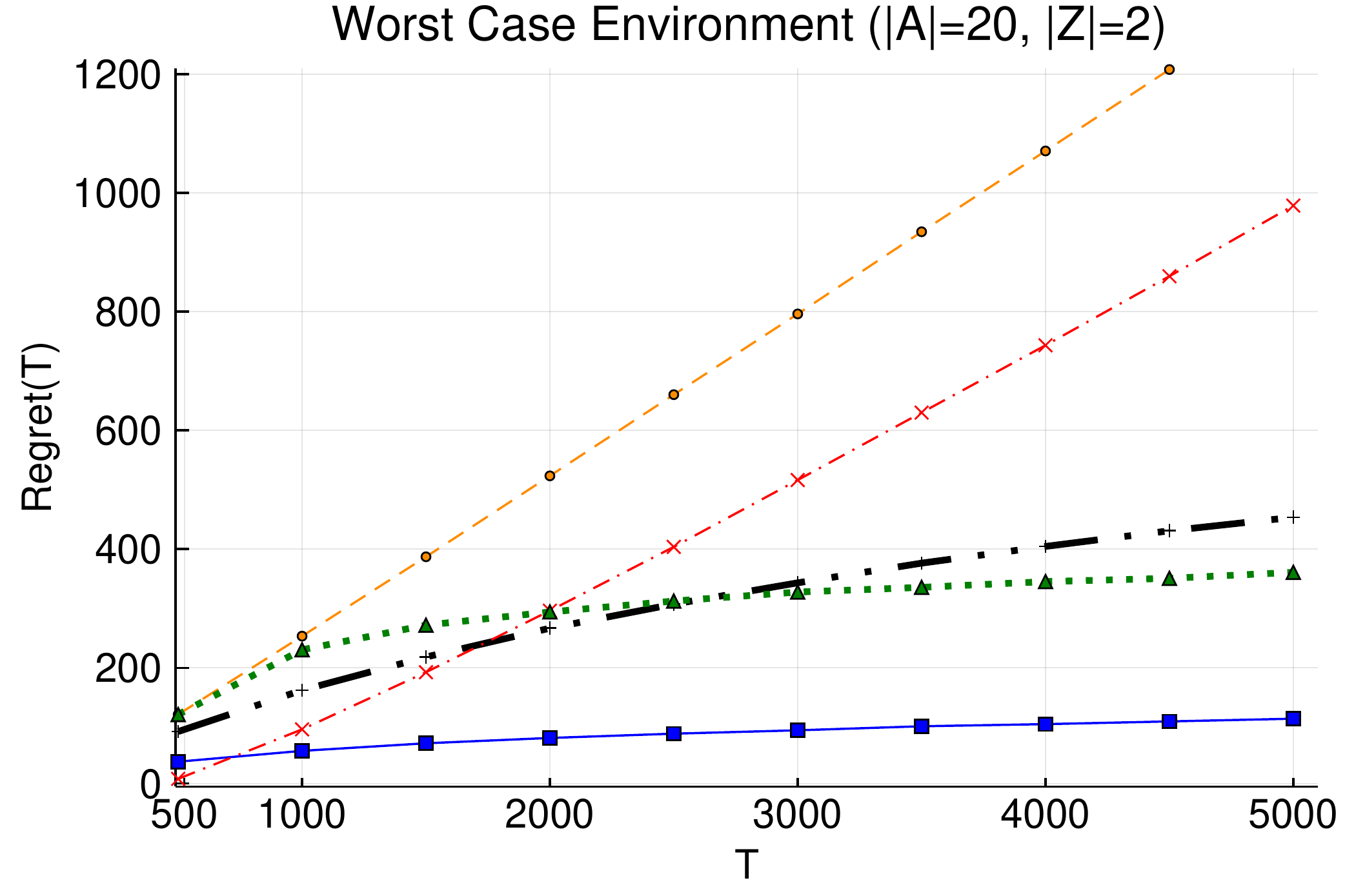}}
\preprint{\includegraphics[scale=0.35]{figures/cb.pdf}}
\preprint{\includegraphics[scale=0.35]{figures/worst.pdf}}
\caption{Regret when the \propertyname{} property holds (left) and when it fails (right).}
\label{fig:sims}
\end{figure}

\subsection{Worst Case \envName{}}

Second, we consider an \envname{} that is \emph{not} \propertyname{}. 
We use the same general \envname{} from our lower bound in \cref{fact:causal-failure}, where the causal algorithms learn a biased estimate of which $\postcontextval\in\{0,1\}$ has a higher conditional mean since $\postcontext{}$ is not a \unobservedname{}, and consequently they concentrate onto a bad \actionname{}. 

In \cref{fig:sims} (right panel), we again observe three main effects: (a) both \CUCB{} and \CUCBt{} incur linear \regretname{}, as prescribed by \cref{fact:causal-failure}, (b) \HCUCB{} achieves sublinear \regretname{}, although it is worse than \UCB{} (we show in \cref{fact:adaptive-impossible-short} that optimal adaptivity is impossible), and (c) \Corral{} does not suffer linear regret, but appears to still do worse than \HCUCB{} for large $T$. 

%% file: section-files/main-sections/adaptivity.tex
\section{Adaptivity for Causal Structures}\label{sec:adaptivity}

Thus far we have analyzed the \regretname{} of various algorithms in two cases: \envnames{} that either are or are not \propertyname{}. Minimax algorithms (\UCB{}) fail to achieve smaller regret for \propertyname{} \envnames{}, while causal algorithms (\CUCB{}) fail catastrophically in \envnames{} that are not \propertyname{}.
In this section, we formalize the notion of strict adaptivity (\emph{adaptive minimax optimality}) with respect to the \propertyname{} property, show that it is \emph{impossible} to be adaptively minimax optimal with respect to the \propertyname{} property, and discuss a relaxed notion of optimal adaptivity based on Pareto optimality.

\subsection{Generic Algorithms}

In order to describe adaptivity (and its impossibility) in the \stochgamename{}, we require a higher level of abstraction than a \policyname{}, which we achieve with \emph{\algonames{}}. 
It is possible to define algorithms and the corresponding notion of adaptivity in abstract generality.
However, we take the same perspective as \citet{bilodeau20semiadv} (who define adaptive minimax optimality with respect to relaxations of the \iid{} assumption), and sacrifice some generality by defining \algonames{} using the specific objects we study in this work.
Formally, an \algoname{} is any map from the problem-specific inputs to the space of compatible \policynames{}, denoted by
\*[
	\algo: (\actionspace, \postcontextspace, T, \priormargenv{}(\postcontext{}))
	\mapsto \algo(\actionspace, \postcontextspace, T, \priormargenv{}(\postcontext{})) \in \policyspace(\actionspace, \postcontextspace, T).
\]
We denote the set of all algorithms by $\pcalgospace$, and the subset of algorithms that are constant in $(\postcontextspace, \priormargenv{}(\postcontext{}))$ by $\mabalgospace$; this subset contains the classical bandit algorithms that are agnostic to knowledge of \postcontextnames{}, or more specifically, do not exploit causal structure.

\subsection{Adaptive Minimax Optimality}\label{sec:adaptive-optimality}

Let
$\propertyfunc: \envspace{} \mapsto \{0,1\}$ encode whether a given \envname{} satisfies a certain property of interest; in this work, it is always an indicator for whether $\env{}$ is \propertyname{}. 
Further, for every $\margdist\in\probspace(\postcontextspace)$, denote the set of all \envnames{} with marginal $\margdist$ by $\envmargspace(\margdist) = \{\env{}\in\envspace: \env{}(\postcontext{}) = \margdist\}$.
There are multiple ways one could define optimal adaptivity: we propose the following notion of strict adaptivity, 
which requires that the \statname{} to do as well as they possibly could have if they had access to $\propertyfunc(\env{})$ in advance, but without this knowledge.

\begin{definition}\label{def:adapt-optimal}
An \algoname{} $\algo\in\pcalgospace$ is \emph{adaptively minimax optimal} with respect to $\propertyfunc$ if and only if there exists $C>0$ such that for all $\actionspace$, $\postcontextspace$, $\margdist\in\probspace(\postcontextspace)^\actionspace$, and $T$,
\noticml{
\[\label{eqn:worst-adapt}
	\sup_{\env{}\in\envmargspace(\margdist)}
	\regretalgo{T}
	\leq 
	C \,
	\inf_{\policy{}\in\policyspace(\actionspace, \postcontextspace, T)} 
	\sup_{\env{}\in\envmargspace(\margdist)}
	\regretboth{T}
\]
}
\icml{
\[\label{eqn:worst-adapt}
	&\sup_{\env{}\in\envmargspace(\margdist)}
	\regretalgo{T}\\
	&\qquad \leq 
	C \,
	\inf_{\policy{}\in\policyspace(\actionspace, \postcontextspace, T)} 
	\sup_{\env{}\in\envmargspace(\margdist)}
	\regretboth{T}
\]
}
and
\noticml{
\[\label{eqn:benign-adapt}
	\sup_{\env{}\in\envmargspace(\margdist),\, \propertyfunc(\env{})=1}
	\regretalgo{T}
	\leq 
	C \,
	\inf_{\policy{}\in\policyspace(\actionspace, \postcontextspace, T)} 
	\sup_{\env{}\in\envmargspace(\margdist),\, \propertyfunc(\env{})=1}
	\regretboth{T}.
\]
}
\icml{
\[\label{eqn:benign-adapt}
	&\sup_{\env{}\in\envmargspace(\margdist),\, \propertyfunc(\env{})=1}
	\regretalgo{T} \\
	&\quad \leq 
	C \,
	\inf_{\policy{}\in\policyspace(\actionspace, \postcontextspace, T)} 
	\sup_{\env{}\in\envmargspace(\margdist),\, \propertyfunc(\env{})=1}
	\regretboth{T}.
\]
}
\end{definition}

\subsection{Impossibility of Strict Adaptivity}\label{sec:adaptive-impossible}

We now show it is \emph{impossible} for any \algoname{} to always realize the benefits of the \propertyname{} property while also recovering the worst-case rate of $\sqrt{\abs{\actionspace} T}$ 
\noticml{(e.g., Theorems~9.1 and 15.2 of \citep{bandit20book}),}
\icml{\citep[e.g., Theorems~9.1 and 15.2 of][]{bandit20book},}
even when the \algoname{} has access to the true marginals.
Our proof strategy is a modification of the finite-time lower bounds from Section~16.2 of \citet{bandit20book}. 
Notably,  the lower bounds of \citet{lu20causal} already imply that any algorithm that does not take advantage of causal structure cannot be adaptively minimax optimal. 
We prove a significantly stronger result: even algorithms that use $\postcontextspace$ and $\priormargenv{}(\postcontext{}) = \env{}(\postcontext{})$ cannot be adaptively minimax optimal!

\begin{theorem}\label{fact:adaptive-impossible-short}
Let $\algo{}\in\pcalgospace$ be such that there exists $C>0$ such that for all $\actionspace$, $\postcontextspace$, and $T$,
\*[
	\sup_{\env{}\in\envspace}
	\regretalgoadapt{T} \leq C\,\sqrt{\abs{\actionspace} T}.
\]
There exists a constant $C'$ such that for all $\actionspace$, $\postcontextspace$, and $T \geq \abs{\actionspace}$, there exists \propertyname{} $\env{}\in\envspace$ with
\*[
	\regretalgoadapt{T}
	\geq 
	C' \sqrt{\abs{\actionspace} T}.
\]
\end{theorem}

\subsection{Pareto-Adaptive Minimax Optimality}\label{sec:pareto-adaptive-optimality}

In light of this impossibility result, it is of interest to characterize relaxations of strict adaptivity that are achievable.
\citet{koulen13pareto} and \citet{lattimore15pareto} introduce the \emph{Pareto frontier} of \regretname{}, which when applied to the \propertyname{} property, is all tuples of \regretname{} guarantees
such that improving the \regretname{} in \propertyname{} \envnames{} would necessarily force the worst-case \regretname{} to increase.
We propose that it is desirable for an \algoname{} to do as well as possible in the worst-case, subject to always realizing smaller \regretname{} on \propertyname{} \envnames{}. 
Formally, let $\adapcalgospace$ be the subset of $\algo\in\pcalgospace$ that satisfy \cref{eqn:benign-adapt} for some constant $C$.

\begin{definition}
An \algoname{} $\adaalgo\in\adapcalgospace$ is \emph{Pareto-adaptively minimax optimal} with respect to $\propertyfunc$ if and only if there exists $C>0$ such that for all $\algo\in\adapcalgospace$, $\actionspace$, $\postcontextspace$, $\margdist\in\probspace(\postcontextspace)^\actionspace$, and $T$,
\noticml{
\*[
	\sup_{\env{}\in\envmargspace(\margdist)}
	\regretadaalgo{T}
	\leq 
	C \,
	\sup_{\env{}\in\envmargspace(\margdist)}
	\regretalgo{T}.
\]
}
\icml{
\*[
	&\sup_{\env{}\in\envmargspace(\margdist)}
	\regretadaalgo{T}\\
	&\qquad \leq 
	C \,
	\sup_{\env{}\in\envmargspace(\margdist)}
	\regretalgo{T}.
\]
}
\end{definition}	

It remains an open problem to prove whether \HCUCB{} is Pareto-adaptively minimax optimal, and more generally to identify the Pareto frontier for the causal bandit problem.

%% file: section-files/main-sections/literature.tex
\section{Related Work}
\label{sec:literature}

\citet{kocaoglu17learning} and \citet{lindgren18learning} efficiently used interventions to learn a causal graph 
under standard causal assumptions \citep{pearl09causality2}.
\citet{hyttinen13causal} identified the minimal number of experiments needed to learn underlying causal structure of variables 
with only a linear structural assumption. 
\citet{lee19structural,lee20intervene,lee20identifiable} identified the minimal set of interventions that permit causal effect identification in the presence of known causal structure, 
while 
\citet{kumor21imitation} studied analogues of the \emph{back-door condition} for identifying causal effects with unobserved confounders.
\citet{bareinboim15causal} introduced causal bandits with a binary motivating example to demonstrate that empirically better performance can be achieved by exploiting a specific, known causal structure. 
\citet{lattimore16causal} and \citet{yabe18propogating} studied best-arm identification, where the \statname{} does not incur any penalty for exploration rounds. 
Given knowledge of the causal graph informing the interventions and response, they separately proved that exponential improvements in the dependence of the \regretname{} on the \actionname{} set are possible provided the underlying distribution on the causal graph is sufficiently ``balanced''. 

\citet{sen17latentcausal} obtained an instance-dependent \regretname{} bound under causal assumptions, but obtained the wrong dependence on the arm gap ($\Delta^{-2}$ rather than $\Delta^{-1}$), and consequently in the worst-case the dependence on $T$ may still dwarf the structural benefits.
\citet{sen17interventions} studied an alternative type of intervention, where rather than fixing a node only the conditional probabilities are changed. This notion is easily stated in our notation, since we allow for abstract families of distributions (indexed by abstract ``interventions'') to define a \envname{}. However, they focused on distribution-dependent guarantees under stronger causal assumptions, and hence our results are not directly comparable. 

All of the above \regretname{} bounds heavily require assumptions about the causal graph, and without such assumptions the presumed information learned from non-intervening rounds can catastrophically mislead the \statname{} in exactly the same way that \CUCB{} suffers in our \cref{fact:causal-failure}. Hence, it remains an interesting open problem to study adaptivity in each of these variations of the causal bandit setting, and our work provides a stepping stone to do so.

Prior to the present work, \citet{lu20causal} has already been extended in multiple directions. 
\citet{dekroon20separating} observed that \CUCB{} can be reduced to requiring only a separating set, but only prove the regret is no worse than that of \UCB{} if a separating set is observed.
The authors remark that a causal discovery algorithm could in principle be used to learn the separating set online, but observed in their experiments that they obtain biased estimates and hence there are no convergence guarantees.
\citet{lu21unknown} replaced knowledge of the causal graph with the assumption that the causal graph has a tree structure, and incorporated the cost of learning the tree into the full \regretname{} bound.
\citet{nair21budget} provided an instance-dependent \regretname{} bound for an alternative algorithm to \CUCB{}, which they call \CUCBt{}, in the presence of the full causal graph. 
While they demonstrated empirically that \CUCBt{} outperforms \CUCB{} for certain instances, we find that \CUCBt{} performs much worse when a \unobservedname{} is observed, and the provable linear lower bound (\cref{fact:causal-failure}) also applies to \CUCBt{} when there are no observed \unobservednames{}.

%% file: section-files/main-sections/discussion.tex
\section{Discussion}\label{sec:discussion}

We have demonstrated that the
improved regret possible when a \unobservedname{} is observed can also be realized in the multi-armed bandit problem by requiring only certain conditional independencies, which we have formalized using the \propertyname{} property.
We proved that it is impossible to optimally adapt to this property, but provided a new algorithm (\HCUCB{}) that simultaneously recovers the improved regret for \propertyname{} \envnames{} and significantly improves on prior work when the \propertyname{} property does not hold.
Crucially, our algorithm requires 
no more assumptions about the world than vanilla bandit algorithms.
We expect our results to spur future work on (a) improved adaptation to the \propertyname{} property, (b) relaxations of the \propertyname{} property for which optimal adaptation is possible, and (c) adaptation in more general partial feedback settings.

In practice, \HCUCB{} will be most useful in settings with a large action space and intermediate variables that may plausibly satisfy the \propertyname{} property. 
In passing, we mention that one such example is learning the causal effect of genome edits (interventions) on disease phenotypes.
Here, the \postcontextname{} could be gene expressions that are sometimes assumed to be a $d$-separator (e.g., \citep{ainsworth16genes}).
We leave the implementation of our algorithm in clinical settings and collaboration with practitioners for future work.

%% file: acknowledgements.tex
\section*{Acknowledgements}\label{sec:acknowledgements}
BB is supported by an NSERC Canada Graduate Scholarship and the Vector Institute.
DMR is supported in part by an NSERC Discovery Grant, Canada CIFAR AI Chair funding through the Vector Institute, and an Ontario Early Researcher Award. 
This material is based also upon work supported by the United States Air Force under Contract No.\ FA850-19-C-0511. Any opinions, findings and conclusions or recommendations expressed in this material are those of the author(s) and do not necessarily reflect the views of the United States Air Force.
Part of this work was done while BB and DMR were visiting the Simons Institute for the Theory of Computing.
We thank Zachary Lipton for suggesting to consider simultaneously adapting to estimation of the marginal distributions, and we thank Chandler Squires, Vasilis Syrgkanis, and Angela Zhou for helpful discussions.

%% file: section-files/main-sections/causal-proofs.tex
\section{Proofs for Causal Equivalences}\label{sec:causal-proofs}

We begin with a restating of standard definitions for completeness.

\subsection{Standard Results from Causal Literature}

To be more explicit about the role of $\reward$ in the causal setting (namely, it may be continuous), we introduce some more notation.
Let $\nodevec = (\noderv{1},\dots,\noderv{\numnodes}, \response{})$ denote random variables each taking values in $\nodespace{i} = \{\nodeval{i}{1},\dots,\nodeval{i}{\numnodevals{i}}\}$ and $[0,1]$ respectively. A \emph{\causalgraphname{}} is any directed acyclic graph (DAG) $\graphset$ defined on the nodes $(\noderv{1},\dots,\noderv{\numnodes},\response{}\,)$ such that (a) $\response{}$ is a leaf node and (b) if there is a directed arrow from $\noderv{i}$ to $\noderv{i'}$, then $i < i'$. 
Let $\graphdists$ denote the set of all probability distributions $\basedist$ on $(\noderv{1},\dots,\noderv{\numnodes},\response{}\,)$ such that the marginal probabilities over $(\noderv{1},\dots,\noderv{\numnodes})$ are all strictly positive. 

\begin{definition}[Markovian parents, Definition 1.2.1 of \citep{pearl09causality2}]
\label{def:markov-parents}
For any $\basedist \in \graphdists$ and $i\in[\numnodes]$, the \emph{Markovian parents of $\noderv{i}$ under $\basedist{}$} is the minimum-cardinality subset $\nodevec' \subseteq (\noderv{1},\dots,\noderv{i-1})$ such that $\noderv{i} \perp (\noderv{1},\dots,\noderv{i-1}) \setminus \nodevec' \setdelim \nodevec'$ under $\basedist$.
We denote this by $\probparents{\basedist}{i}$.
Similarly, $\probparents{\basedist}{\response{}}$ is the minimum-cardinality subset $\nodevec' \subseteq \nodevec$ such that $\response{} \perp \nodevec \setminus \nodevec' \setdelim \nodevec'$ under $\basedist$.
\end{definition}

\begin{definition}[Graphical parents]
\label{def:graph-parents}
For any \causalgraphname{} $\graphset$, 
the \emph{graphical parents of $\noderv{i}$ under $\graphset$} is the unique subset $\nodevec'\subseteq\nodevec$ of variables that have a directed arrow into $\noderv{i}$. We denote this by $\graphparents{\graphset}{i}$.
Similarly, $\graphparents{\graphset}{\response{}}$ is the unique subset $\nodevec'\subseteq\nodevec$ of variables that have a directed arrow into $\response{}$.
\end{definition}

\begin{definition}[Markov relative, Theorem 1.2.7 of \citep{pearl09causality2}]
\label{def:markov-relative}
A distribution $\basedist\in\graphdists$ is \emph{Markov relative to a \causalgraphname{} $\graphset$} if any only if for all $\noderv{}\in\nodevec\cup\{\response{}\,\}$, $\probparents{\basedist{}}{\noderv{}} \subseteq \graphparents{\graphset}{\noderv{}}$.
\end{definition}

\begin{remark}[Equation 1.33 of \citep{pearl09causality2}]\label{fact:markov-factoring}
If $\basedist\in\graphdists$ is Markov relative to $\graphset$, then for all measurable $\borelset \subseteq \responsespace$ and $(j_{1},\dots,j_{\numnodes}) \in \prod_{i=1}^\numnodes [\numnodevals{i}]$,
\*[
	\basedist(\response{}\in\borelset, \noderv{1}=\nodeval{1}{j_1},\dots,\noderv{\numnodes}=\nodeval{\numnodes}{j_\numnodes})
	= \basedist(\response{}\in\borelset \setdelim \graphparents{\graphset}{\response{}}=\dummyvecval_{\response{}})
	\prod_{i=1}^\numnodes \basedist(\noderv{i}=\nodeval{i}{j_i} \setdelim \graphparents{\graphset}{i}=\dummyvecval_i),
\]
where the conditioning is understood to be on the event where the parents take the specific values defined by $\dummyvecval = (\nodeval{1}{j_1},\dots,\nodeval{\numnodes}{j_\numnodes})$.
\end{remark}

\begin{definition}[Causal intervention, Definition 1.3.1 of \citep{pearl09causality2}]\label{def:causal-intervention}
Let $\basedist\in\graphdists$ be Markov relative to $\graphset$. The \emph{interventional distribution} induced on $\nodevec$ by the \emph{intervention} $\actionval = \doaction(\noderv{i_1} = \nodeval{i_1}{j_{i_1}},\dots,\noderv{i_\ell} = \nodeval{i_\ell}{j_{i_\ell}})$ is
\*[
	&\hspace{-1em}\intervenedist{\actionval}(\response{}\in\borelset, \noderv{1}=\nodeval{1}{j_1},\dots,\noderv{\numnodes}=\nodeval{\numnodes}{j_\numnodes}) \\
	&= \ind\{\noderv{i_1} = \nodeval{i_1}{j_{i_1}},\dots,\noderv{i_\ell} = \nodeval{i_\ell}{j_{i_\ell}}\}
	\basedist(\response{}\in\borelset \setdelim \graphparents{\graphset}{\response{}} = \dummyvecval_{\response{}})
	\prod_{i\not\in\{i_1,\dots,i_\ell\}} \basedist(\noderv{i}=\nodeval{i}{j_i} \setdelim \graphparents{\graphset}{i} = \dummyvecval_i).
\] 
\end{definition}

\begin{definition}[$d$-Separated, Definition~2.4.1 of \citep{pearl16primer}]
\label{def:d-sep}
\emph{$\condvec$ $d$-separates $\response{}$ from $\intervenevec$ (on $\graphset$)} if and only if every path between $\intervenevec$ and $\response{}$ is \emph{blocked}; that is, every path contains either (a) $\bigcirc \rightarrow B \rightarrow \bigcirc $ or $\bigcirc \leftarrow B \rightarrow \bigcirc $ such that $B \in \condvec$, or (b) $\bigcirc \rightarrow B \leftarrow \bigcirc $ with no descendents of $B$ (including itself) in $\condvec$.
\end{definition}

\begin{definition}[Back-Door Path, Section~3.3.1 of \citep{pearl09causality2}]
A path from $\condvec$ to $\condvec'$ is a \emph{back-door path} if it begins with an arrow directed into $\condvec$.
\end{definition}

\begin{definition}[Front-Door Criterion, Definition~3.3.3 of \citep{pearl09causality2}]
\label{def:front-door}
\emph{$\condvec$ satisfies the front-door criterion relative to $(\intervenevec,\response{})$} on $\graphset$ if and only if (a) all directed paths from $\intervenevec$ to $\response{}$ pass through $\condvec$, (b) there is no unblocked back-door path from $\intervenevec$ to $\condvec$, and (c) all back-door paths from $\condvec$ to $\response{}$ are blocked by $\intervenevec$.
\end{definition}

\subsection{Proof of \cref{fact:benign-equals-dsep}}

We first state an intuitive result about $d$-separation that is used often in the causal literature, but we could not find stated or proved precisely as follows.
\begin{lemma}\label{fact:dsep-parents}
If $\condvec$ $d$-separates $\response{}$ from $\intervenevec$ and $\graphparents{\graphset}{\intervenevec} \subseteq \nodevec\setminus\condvec$, then $\condvec$ $d$-separates $\response{}$ from $(\intervenevec, \graphparents{\graphset}{\intervenevec})$.
\end{lemma}
\begin{proof}[Proof of \cref{fact:dsep-parents}]
First, every path from $\graphparents{\graphset}{\intervenevec}$ to $\response{}$ that passes through $\intervenevec$ must satisfy one of (a) or (b) since a subpath does. Further, every path from $\graphparents{\graphset}{\intervenevec}$ to $\response{}$ that doesn't pass through $\intervenevec$ can be extended to a (back-door) path from $\intervenevec$ to $\response{}$ using the edge $\graphparents{\graphset}{\intervenevec} \rightarrow \intervenevec$, but $\graphparents{\graphset}{\intervenevec} \subseteq \nodevec\setminus\condvec$ (and hence this part of the path cannot satisfy either property), so the original path must satisfy either (a) or (b).
\end{proof}

Next, we recall the probabilistic equivalence of $d$-separation. 

\begin{theorem}[Theorem 1.2.4 of \citep{pearl09causality2}]
Fix a \causalgraphname{} $\graphset$ and two disjoint sets $\condvec \subseteq \nodevec$ and $\intervenevec \subseteq \nodevec$ such that $\graphparents{\graphset}{\intervenevec} \subseteq \nodevec\setminus\condvec$. Then $\condvec$ $d$-separates $\response{}$ from $\intervenevec$ if and only if $\response{} \perp \intervenevec \setdelim \condvec$ under every distribution $\basedist\in\graphdists$ that is Markov relative to $\graphset$.
\end{theorem}

We now turn to the main argument to prove \cref{fact:benign-equals-dsep}.
Let $\basedist\in\graphdists$ and $\actionval=\doaction(\intervenevec=\intervenevecval)$ be arbitrary.
First, we prove the ``Causal Effect Rule'' \citep{pearl16primer}.
For any $\nodevec'\subseteq\nodevec$,
\*[
	\intervenedist{\actionval}(\nodevec' = \nodevecval') 
	&= \sum_{\dummyvecval} \intervenedist{\actionval}(\nodevec' = \nodevecval' \setdelim \graphparents{\graphset}{\intervenevec} = \dummyvecval) \intervenedist{\actionval}(\graphparents{\graphset}{\intervenevec} = \dummyvecval) \\
	&= \sum_\dummyvecval \basedist(\nodevec' = \nodevecval' \setdelim \intervenevec=\intervenevecval, \graphparents{\graphset}{\intervenevec} = \dummyvecval) \basedist(\graphparents{\graphset}{\intervenevec} = \dummyvecval),
\]
where the sum is over all possible values that $\graphparents{\graphset}{\intervenevec}$ can take and we have used that (a) conditional on $\graphparents{\graphset}{\intervenevec}$, the interventional and conditional distributions given $\intervenevec=\intervenevecval$ are equivalent, and (b) the marginal distribution of $\graphparents{\graphset}{\intervenevec}$ is unchanged by intervening on $\intervenevec$. 

Now, suppose $\condvec$ $d$-separates $\response{}$ from $\intervenevec$.
Then, it follows that
\*[
	\intervenedist{\actionval}(\response{} \in \borelset \setdelim \condvec = \condvecval)
	&= \frac{\intervenedist{\actionval}(\response{} \in \borelset, \condvec = \condvecval)}{\intervenedist{\actionval}(\condvec = \condvecval)} \\
	&= \frac{\sum_\dummyvecval \basedist(\response{}\in\borelset, \condvec=\condvecval \setdelim \intervenevec=\intervenevecval, \graphparents{\graphset}{\intervenevec} = \dummyvecval) \basedist(\graphparents{\graphset}{\intervenevec} = \dummyvecval)}{\sum_\dummyvecval \basedist(\condvec=\condvecval\setdelim \intervenevec=\intervenevecval, \graphparents{\graphset}{\intervenevec} = \dummyvecval) \basedist(\graphparents{\graphset}{\intervenevec} = \dummyvecval)} \\
	&= \frac{\sum_\dummyvecval \basedist(\response{}\in\borelset \setdelim \condvec=\condvecval, \intervenevec=\intervenevecval, \graphparents{\graphset}{\intervenevec} = \dummyvecval) \basedist(\condvec=\condvecval \setdelim \intervenevec=\intervenevecval, \graphparents{\graphset}{\intervenevec} = \dummyvecval) \basedist(\graphparents{\graphset}{\intervenevec} = \dummyvecval)}{\sum_\dummyvecval \basedist(\condvec=\condvecval\setdelim \intervenevec=\intervenevecval, \graphparents{\graphset}{\intervenevec} = \dummyvecval) \basedist(\graphparents{\graphset}{\intervenevec} = \dummyvecval)} \\
	&= \frac{\sum_\dummyvecval \basedist(\response{}\in\borelset \setdelim \condvec=\condvecval) \basedist(\condvec=\condvecval \setdelim \intervenevec=\intervenevecval, \graphparents{\graphset}{\intervenevec} = \dummyvecval) \basedist(\graphparents{\graphset}{\intervenevec} = \dummyvecval)}{\sum_\dummyvecval \basedist(\condvec=\condvecval\setdelim \intervenevec=\intervenevecval, \graphparents{\graphset}{\intervenevec} = \dummyvecval) \basedist(\graphparents{\graphset}{\intervenevec} = \dummyvecval)} \\
	&= \basedist(\response{}\in\borelset \setdelim \condvec=\condvecval),
\]
where the second last step uses \cref{fact:dsep-parents}.

Conversely, suppose there exists $\basedist\in\graphdists$ that is Markov relative to $\graphset$ and under which $\response{} \not\perp \intervenevec \setdelim \condvec$. 
This implies 
\noticml{(see the remark following Theorem~1.2.4 in \citep{pearl09causality2})}
\icml{\citep[see the remark following Theorem~1.2.4 in][]{pearl09causality2}}
there exists $\basedist\in\graphdists$ with $\basedist(\response{}\in\borelset \setdelim \condvec=\condvecval, \intervenevec=\intervenevecval, \graphparents{\graphset}{\intervenevec} = \dummyvecval) \neq \basedist(\response{}\in\borelset \setdelim \condvec=\condvecval)$.
By the above, this implies that $\intervenedist{\actionval}(\response{} \in \borelset \setdelim \condvec = \condvecval) \neq \basedist(\response{}\in\borelset \setdelim \condvec=\condvecval)$.
\manualendproof

\subsection{Proof of \cref{fact:benign-equals-dsep-null}}

If $\condvec$ $d$-separates $\response{}$ from $\intervenevec$ on $\graphset_{\bar\intervenevec}$, then by \cref{fact:benign-equals-dsep} $\{\intervenedist{\actionval}(\condvec, \response{}\,): \actionval \in \actionspace\}$ is \propertyname{} for every $\basedist\in\graphdists$ that is Markov relative to $\graphset_{\bar\intervenevec}$, and hence is still \propertyname{} when the \nullaction{} \genactionname{} is excluded. It remains to observe that for any $\basedist$ that is Markov relative to $\graphset$, there exists $\basedist'$ that is Markov relative to $\graphset_{\bar\intervenevec}$ such that
\*[
	\{\intervenedistnull{\actionval}(\condvec, \response{}\,): \actionval \in \actionspacenull\}
	= \{\intervenedist{\actionval}(\condvec, \response{}\,): \actionval \in \actionspacenull\}.
\]

Conversely, suppose there exists $\basedist\in\graphdists$ that is Markov relative to $\graphset_{\bar\intervenevec}$ and under which $\response{} \not\perp \intervenevec \setdelim \condvec$. Since necessarily $\intervenedist{\actionval}(\response{} \in \borelset \setdelim \condvec = \condvecval) = \basedist(\response{}\in\borelset \setdelim \condvec=\condvecval)$ when $\actionval$ is the \nullaction{} \genactionname{}, it must be some $\actionval\in\actionspacenull$ that realizes the failure from the proof of \cref{fact:benign-equals-dsep}, which means $\{\intervenedist{\actionval}(\condvec, \response{}\,): \actionval \in \actionspacenull\}$ is not \propertyname{}.
\manualendproof

\subsection{Proof of \cref{fact:front-implies-dsep}}
Suppose that $\condvec$ satisfies the front-door criterion relative to $(\intervenevec, \response{})$ on $\graphset$ and there exists a path from $\intervenevec$ to $\response{}$ on $\graphset_{\bar\intervenevec}$ that is unblocked by $\condvec$. The path cannot be directed, since by the front-door criterion (a) it would include the subpath $\bigcirc \rightarrow \condvec \rightarrow \bigcirc$, and hence would be blocked. The path also cannot be a back-door path since there are no arrows going into $\intervenevec$ on $\graphset_{\bar\intervenevec}$. Thus, there must be some part of the path that is of the form $\bigcirc \rightarrow B \leftarrow \bigcirc $ for some variable $B \in \nodevec$ and there are no remaining colliders on the subpath from $B$ to $\response{}$. Since the path is unblocked, this $B$ must have a descendant (potentially itself) in $\condvec$, and this creates a back-door path from $\condvec$ to $\response{}$. On the portion of the back-door path that is from $\condvec$ to $B$, there can be no colliders since then $\condvec$ would not be a descendant of $B$, and hence this backdoor path contains no colliders. Since the back-door path also does not contain $\intervenevec$, it is unblocked by $\intervenevec$, which violates the front-door criterion (c). Thus, no path from $\intervenevec$ to $\response{}$ on $\graphset_{\bar\intervenevec}$ that is unblocked by $\condvec$ can exist, so $\condvec$ blocks every path and hence $\condvec$ $d$-separates $\response{}$ from $\intervenevec$ on $\graphset_{\bar\intervenevec}$.
\manualendproof

%% file: section-files/main-sections/main-proofs.tex
\section{Proofs for Regret Bounds}\label{sec:main-proofs}

\subsection{Concentration of Empirical Means}\label{sec:proof-concentrate}

For a fixed $t\in[T]$, $\actionval\in\actionspace$, and $\postcontextval\in\postcontextspace$, define the events 
\*[
	\actionconc{t}(\actionval) = \bigg\{\abs[2]{\estcondmeanaction{t}(\actionval) - \EE_{\env{\actionval}}[\response{}\,]} \leq \sqrt{\frac{\log(2/\highprobparam{})}{2\numobsaction{t}(\actionval)}} \, \bigg\}
\]
and
\*[
	\contextconc{t}(\postcontextval) = \bigg\{\max_{\actionval\in\actionspace}\abs[2]{\estcondmeancontext{t}(\postcontextval) - \EE_{\env{\actionval}}[\response{} \setdelim \postcontext{} = \postcontextval]} \leq \sqrt{\frac{\log(2/\highprobparam{})}{2\numobscontext{t}(\postcontextval)}} \, \bigg\}.
\]
Let $\actionconc{} = \cap_{t\in[T],\actionval\in\actionspace} \actionconc{t}(\actionval)$, $\contextconc{} = \cap_{t\in[T],\postcontextval\in\postcontextspace} \contextconc{t}(\postcontextval)$, and $\concevent = \actionconc{} \cap \contextconc{}$. 
Finally, define the event
\*[
	\margconc
	=
	\bigg\{
	\sup_{\actionval\in\actionspace}
	\sum_{\postcontextval\in\postcontextspace} \abs[2]{\PPestenv{\actionval}[\postcontext{}=\postcontextval] - \PPenv{\actionval}[\postcontext{}=\postcontextval]}
	\leq \frac{\sqrt{\abs{\actionspace}\abs{\postcontextspace}\log T}}{T^{1/4}}
	\bigg\}.
\]

\begin{lemma}\label{fact:action-event-concentration}
For any $\env{} \in \envspace$ and $\policy{}\in\policyspace(\actionspace,\postcontextspace,T)$,
\*[
	\PPboth[(\actionconc{})\complement] \leq \abs{\actionspace} T \highprobparam{}.
\]
\end{lemma}
\begin{proof}[Proof of \cref{fact:action-event-concentration}]
For each $\actionval\in\actionspace$, define the new \iid{}\ random variables $\genreward_1(\actionval),\dots,\genreward_T(\actionval) \sim \env{\actionval}$. For any $t\in[T]$, Hoeffding's inequality can be applied to obtain
\*[
	\PP_{\env{\actionval}}\left(\abs[2]{\frac{1}{t}\sum_{s=1}^{t} \genreward_s(\actionval) - \EE_{\env{\actionval}}[\response{}\,]} > \sqrt{\frac{\log(2/\highprobparam{})}{2t}} \right) \leq \highprobparam{}.
\]
Then, using the \iid{}\ property of $(\responserv{1}{\actionval},\dots,\responserv{T}{\actionval})$, 
\*[
	&\hspace{-1em}\PPboth\left(\exists t\in[T], \actionval\in\actionspace: \quad \abs{\estcondmeanaction{t}(\actionval) - \EE_{\env{\actionval}}[\response{}\,]} > \sqrt{\frac{\log(2/\highprobparam{})}{2\numobsaction{t}(\actionval)}} \right) \\
	&\leq \sum_{t=1}^T \sum_{\actionval\in\actionspace} \PP_{\env{\actionval}}\left(\abs[2]{\frac{1}{t}\sum_{s=1}^{t} \genreward_s(\actionval) - \EE_{\env{\actionval}}[\response{}\,]} > \sqrt{\frac{\log(2/\highprobparam{})}{2t}} \right) \\
	&\leq \abs{\actionspace} T \highprobparam{},
\]
where we have used a union bound over $\actionval\in\actionspace$ and $t\in[T]$.
\end{proof}

\begin{lemma}\label{fact:context-event-concentration}
For any $\env{} \in \envspace$ that is \propertyname{}
and $\policy{}$,
\*[
	\PPboth[(\contextconc{})\complement] \leq \abs{\postcontextspace} T\highprobparam{}.
\]
\end{lemma}
\begin{proof}[Proof of \cref{fact:context-event-concentration}]
Since $\env{}$ is \propertyname{}, there exists $\basedist\in\probspace(\postcontextspace\times\responsespace)$ such that for each $\actionval\in\actionspace$, $\env{\actionval}(\response{}\setdelim\postcontext{}) = \basedist(\response{}\setdelim\postcontext{})$. Fix $\postcontextval\in\postcontextspace$, and define joint the distribution $\basecontextdist_\postcontextval(\response{},\postcontext{})=\basedist(\response{}\setdelim\postcontext{})\ind\{\postcontext{}=\postcontextval\}$.
Finally, define the new \iid{}\ random variables $\gencontextreward_1,\dots,\gencontextreward_T \sim \basecontextdist_\postcontextval$.
For any $t\in[T]$, Hoeffding's inequality can be applied to obtain
\*[
	\PP_{\basecontextdist_\postcontextval}\left(\abs[2]{\frac{1}{t}\sum_{s=1}^{t} \gencontextreward_s - \EE_{\basecontextdist_\postcontextval}[\response{}\,]} \leq \sqrt{\frac{\log(2/\highprobparam{})}{2t}} \right)\leq \highprobparam{}.
\]

Then, 
\*[
	&\hspace{-1em}\PPboth\left(\exists t\in[T], \postcontextval\in\postcontextspace: \quad \max_{\actionval\in\actionspace}\abs[2]{\estcondmeancontext{t}(\postcontextval) - \EE_{\env{\actionval}}[\response{} \setdelim \postcontext{} = \postcontextval]} > \sqrt{\frac{\log(2/\highprobparam{})}{2\numobscontext{t}(\postcontextval)}} \right) \\
	&= \PPboth \left(\exists t\in[T], \postcontextval\in\postcontextspace: \quad \abs[2]{\estcondmeancontext{t}(\postcontextval) - \EE_{\basedist}[\response{}\setdelim\postcontext{} = \postcontextval]} \leq \sqrt{\frac{\log(2/\highprobparam{})}{2\numobscontext{t}(\postcontextval)}} \right) \\
	&\leq \sum_{t=1}^T \sum_{\postcontextval\in\postcontextspace} \PP_{\basecontextdist_\postcontextval}\left(\abs[2]{\frac{1}{t}\sum_{s=1}^{t} \gencontextreward_s - \EE_{\basecontextdist_\postcontextval}[\response{}\,]} \leq \sqrt{\frac{\log(2/\highprobparam{})}{2t}} \right) \\
	&\leq \abs{\postcontextspace} T \highprobparam{}.
\]
where we have used a union bound over $\postcontextval\in\postcontextspace$ and $t\in[T]$.
\end{proof}

\begin{theorem}[Theorem~1 of \citet{cannone20discrete}]\label{fact:multinomial}
Let $p$ be any distribution on $[k]$ for some integer $k$.
For any $\eps,\delta>0$, if $n \geq \max\{k/\eps^2, (2/\eps^2)\log(2/\delta)\}$ 
and $X_1,\dots,X_n$ is an i.i.d.\ sample from $p$, then the MLE estimator 
\*[
	\hat p_n(j) = \frac{1}{n} \sum_{t=1}^n \ind\{X_t = j\} \quad \forall j\in[k]
\] 
satisfies
\*[
	\PP\Big[\frac{1}{2}\sum_{j\in[k]}\abs{\hat p_n(j) - p(j)} > \eps \Big] \leq \delta.
\]
\end{theorem}

\begin{lemma}\label{fact:marginal-event-concentration}
If $\estmargenv{}(\postcontext{})$ is estimated using uniform exploration of at least $\lceil 4 \sqrt{T} / \abs{\actionspace} \rceil$ rounds for each $\actionval\in\actionspace$,
\*[
	\PPboth[(\margconc)\complement] \leq 2\abs{\actionspace} / T.
\]
\end{lemma}
\begin{proof}
Let $\eps = (1/2)T^{-1/4}\sqrt{\abs{\actionspace}\abs{\postcontextspace}\log T}$, $\delta = 2/T$, and $n$ denote the number of exploration rounds used to estimate each $\estmargenv{\actionval}$. By a union bound and \cref{fact:multinomial}, if $n \geq \max\{\abs{\postcontextspace}/\eps^2, (2/\eps^2)\log(2/\delta)\}$ then
\*[
	\PPboth[(\margconc)\complement]
	&\leq \PPboth\Bigg[\exists \actionval\in\actionspace: \sum_{\postcontextval\in\postcontextspace} \abs[2]{\PPestenv{\actionval}[\postcontext{}=\postcontextval] - \PPenv{\actionval}[\postcontext{}=\postcontextval]}
	> 2\eps\Bigg] \\
	&\leq \sum_{\actionval\in\actionspace}
	\PPboth\Bigg[\frac{1}{2}\sum_{\postcontextval\in\postcontextspace} \abs[2]{\PPestenv{\actionval}[\postcontext{}=\postcontextval] - \PPenv{\actionval}[\postcontext{}=\postcontextval]}
	> \eps\Bigg] \\
	&\leq 2\abs{\actionspace} / T.
\]

Then, it remains to observe that when $T \geq 3$ and $\abs{\postcontextspace} \geq 2$ (which can be trivially assumed, since $\estmargenv{}$ is known exactly if $\abs{\postcontextspace}=1$),
\*[
	\frac{\abs{\postcontextspace}}{\eps^2}
	&= \frac{4\abs{\postcontextspace}\sqrt{T}}{\abs{\actionspace}\abs{\postcontextspace}\log T}
	\leq \frac{4\sqrt{T}}{\abs{\actionspace}}
\]
and
\*[
	\frac{2}{\eps^2}\log(2/\delta)
	&= \frac{8\sqrt{T}}{\abs{\actionspace}\abs{\postcontextspace}\log T} (\log T)
	\leq \frac{4\sqrt{T}}{\abs{\actionspace}}.
\]
\end{proof}

\subsection{Bounding Accumulated Regret}\label{sec:proof-accumulate}

\begin{lemma}\label{fact:accumulated-actions}
For any $\env{} \in \envspace$, $\policy{}\in\policyspace(\actionspace,\postcontextspace,T)$, and $t<t'\in[T]$, it holds almost surely that
\*[
	\sum_{s=t}^{t'} \frac{1}{\sqrt{\numobsaction{s-1}(\actionrv{s})}}
	\leq \sqrt{8\abs{\actionspace}(t'-t)}\,.
\]
\end{lemma}
\begin{proof}[Proof of \cref{fact:accumulated-actions}]
Using Lemma~4.13 of \citet{orabona19book},
\*[
	\sum_{s=t}^{t'}  \frac{1}{\sqrt{\numobsaction{s-1}(\actionrv{s})}}
	&= 
	\sum_{s=t}^{t'}
	\sum_{\actionval\in\actionspace} 
	\frac{\ind\{\actionrv{s}=\actionval\}}{\sqrt{1 \vee \sum_{j=1}^{s-1} \ind\{\actionrv{j}=\actionval\}}} \\
	&\leq
	\sum_{s=t}^{t'}
	\sum_{\actionval\in\actionspace} 
	\frac{\ind\{\actionrv{s}=\actionval\}}{\sqrt{\sum_{j=1}^{t-1} \ind\{\actionrv{j}=\actionval\} + (1/2)\sum_{j=t}^{s} \ind\{\actionrv{j}=\actionval\}}} \\
	&\leq
	\sqrt{2} \sum_{\actionval\in\actionspace} 
	\int_{\sum_{j=1}^{t-1} \ind\{\actionrv{j}=\actionval\}}^{\sum_{j=1}^{t'} \ind\{\actionrv{j}=\actionval\}}
	x^{-1/2} \dee x \\
	&= \sqrt{8}  \sum_{\actionval\in\actionspace} \left(\sqrt{\sum_{j=1}^{t'} \ind\{\actionrv{j}=\actionval\}} - \sqrt{\sum_{j=1}^{t-1} \ind\{\actionrv{j}=\actionval\}} \right) \\
	&\leq \sum_{\actionval\in\actionspace} \sqrt{8\sum_{j=t}^{t'} \ind\{\actionrv{j}=\actionval\}} \\
	&\leq \abs{\actionspace} \sqrt{\frac{8}{\abs{\actionspace}} \sum_{\actionval\in\actionspace}\sum_{j=t}^{t'} \ind\{\actionrv{j}=\actionval\}} \\
	&= \sqrt{8\abs{\actionspace} (t'-t)}.
\]
\end{proof}

\begin{lemma}\label{fact:accumulated-contexts}
For any $\env{} \in \envspace$, $\policy{}\in\policyspace(\actionspace,\postcontextspace,T)$, and $t<t'\in[T]$, 
\*[
	\EEboth \Bigg[\sum_{s=t}^{t'} \sum_{\postcontextval\in\postcontextspace} \frac{1}{\sqrt{\numobscontext{s-1}(\postcontextval)}}\, \PPenv{\actionrv{s}}[\postcontext{}=\postcontextval] \Bigg]
	\leq
	\sqrt{8\abs{\postcontextspace} (t'-t)}
	+ \sqrt{(1/2)(t'-t)\log(t'-t)} + 2.
\]
\end{lemma}
\begin{proof}[Proof of \cref{fact:accumulated-contexts}]
First,
\*[
	&\hspace{-1em}\EEboth \Bigg[\sum_{s=t}^{t'} \sum_{\postcontextval\in\postcontextspace} \sqrt{1/ \numobscontext{s-1}(\postcontextval)}\, \PPenv{\actionrv{s}}[\postcontext{}=\postcontextval] \Bigg] \\
	&= \EEboth \Bigg[\sum_{s=t}^{t'} \sum_{\postcontextval\in\postcontextspace} \sqrt{1/ \numobscontext{s-1}(\postcontextval)} \, \ind\{\postcontextrv{s}{\actionrv{s}}=\postcontextval\}\Bigg] \\ 
	&\qquad + \EEboth \Bigg[\sum_{s=t}^{t'} \sum_{\postcontextval\in\postcontextspace} \sqrt{1/ \numobscontext{s-1}(\postcontextval)}\, \Big(\PPenv{\actionrv{s}}[\postcontext{}=\postcontextval] - \ind\{\postcontextrv{s}{\actionrv{s}}=\postcontextval\}\Big)\Bigg] \\
	&\leq \sqrt{8\abs{\postcontextspace} (t'-t)} + \EEboth \Bigg[\sum_{s=t}^{t'} \sum_{\postcontextval\in\postcontextspace} \sqrt{1/ \numobscontext{s-1}(\postcontextval)}\, \Big(\PPenv{\actionrv{s}}[\postcontext{}=\postcontextval] - \ind\{\postcontextrv{s}{\actionrv{s}}=\postcontextval\}\Big)\Bigg],
\]
where we have used the same argument as \cref{fact:accumulated-actions} applied to $\numobscontext{s-1}(\postcontextval)$ rather than $\numobsaction{s-1}(\actionval)$. 
Following the analysis of \citet{lu20causal}, for $t \leq j \leq t'$ define the random variable
\*[
	\martingalerv{j} = \sum_{s=t}^j \sum_{\postcontextval\in\postcontextspace} \sqrt{1/ \numobscontext{s-1}(\postcontextval)}\, \Big(\PPenv{\actionrv{s}}[\postcontext{}=\postcontextval] - \ind\{\postcontextrv{s}{\actionrv{s}}=\postcontextval\}\Big).
\]
Then, 
\*[
	\EEboth[\martingalerv{j} \setdelim\actionrv{j}, \historyrv{j-1}]
	=\martingalerv{j-1} + \sum_{\postcontextval\in\postcontextspace} \sqrt{1/ \numobscontext{j-1}(\postcontextval)}\, \EEboth\Big[\PPenv{\actionrv{s}}[\postcontext{}=\postcontextval] - \ind\{\postcontextrv{j}{\actionrv{j}}=\postcontextval\} \setdelim \actionrv{j}\Big] 
	= \martingalerv{j-1}.
\]
Further, it holds almost surely that
\*[
	\abs{\martingalerv{j}-\martingalerv{j-1}}
	&= \abs{\sum_{\postcontextval\in\postcontextspace} \sqrt{1/ \numobscontext{j-1}(\postcontextval)}\, \Big[\PPenv{\actionrv{s}}[\postcontext{}=\postcontextval] - \ind\{\postcontextrv{j}{\actionrv{j}}=\postcontextval\} \Big]} \\
	&= \abs{\EEboth\Bigg[\sum_{\postcontextval\in\postcontextspace} \sqrt{1/ \numobscontext{j-1}(\postcontextval)} \, \ind\{\postcontextrv{j}{\actionrv{j}}=\postcontextval\} \Bigsetdelim \actionrv{j}, \historyrv{j-1}\Bigg] - \sum_{\postcontextval\in\postcontextspace} \sqrt{1/ \numobscontext{j-1}(\postcontextval)} \, \ind\{\postcontextrv{j}{\actionrv{j}}=\postcontextval\}} \\
	&= \abs{\EEboth\Bigg[\sqrt{1/ \numobscontext{j-1}(\postcontextrv{j}{\actionrv{j}})} \setdelim \actionrv{j}, \historyrv{j-1} \Bigg] - \sqrt{1/ \numobscontext{j-1}(\postcontextrv{j}{\actionrv{j}})}} \\
	&\leq 1.
\]
Then, by Azuma-Hoeffding, 
for all $x > 0$
\*[
	\PPboth\Bigg[\abs[0]{\martingalerv{t'}} > \sqrt{x(t'-t)\log (t'-t)}\Bigg] \leq 2e^{-2x\log(t'-t)}.
\]
Thus, since $\abs{\martingalerv{t'}} \leq t'-t$,
\*[
	&\hspace{-1em}\EEboth \Bigg[\sum_{s=t}^{t'} \sum_{\postcontextval\in\postcontextspace} \sqrt{1/ \numobscontext{t-1}(\postcontextval)}\, \Big(\PPenv{\actionrv{s}}[\postcontext{}=\postcontextval] - \ind\{\postcontextrv{s}{\actionrv{s}}=\postcontextval\}\Big)\Bigg] \\
	&\leq 2(t'-t)e^{-2x\log(t'-t)} + \sqrt{x (t'-t) \log (t'-t)}.
\]
Taking $x=1/2$ gives the result.
\end{proof}

\subsection{Proof of \cref{fact:existing-causal}}\label{sec:proof-existing-causal}

First, by \cref{fact:context-event-concentration}
\[\label{eqn:c-ucb-bound0}
	\regretc{T}
	&=\EEbothc \sum_{t=1}^T \Big[\EEenv{\optactionval}[\response{} \,] - \EEenv{\actionrvc{t}}[\response{} \,] \Big] \\
	&\leq \abs{\postcontextspace} T^2 \highprobparam{} + \EEbothc \Big[ \ind\{\contextconc{}\}\sum_{t=1}^T \Big(\EEenv{\optactionval}[\response{} \,] - \EEenv{\actionrvc{t}}[\response{} \,]\Big) \Big].
\]

Then, by 
the \propertyname{} property,
\*[
	&\hspace{-1em}\EEbothc \Big[ \ind\{\contextconc{}\}\sum_{t=1}^T \Big(\EEenv{\optactionval}[\response{} \,] - \EEenv{\actionrvc{t}}[\response{} \,]\Big) \Big] \\
	&= \EEbothc \Big[ \ind\{\contextconc{}\}\sum_{t=1}^T \Big(\EEenv{\optactionval}[\response{} \,] - \pseudoucb{t}(\actionrvc{t}) + \pseudoucb{t}(\actionrvc{t}) - \EEenv{\actionrvc{t}}[\response{} \,]\Big) \Big] \\
	&= \EEbothc \Big[ \ind\{\contextconc{}\}\sum_{t=1}^T \Big(\sum_{\postcontextval\in\postcontextspace}\EEenv{\optactionval}[\response{} \setdelim \postcontext{} = \postcontextval] \PPenv{\optactionval}[\postcontext{}=\postcontextval] - \pseudoucb{t}(\actionrvc{t}) \Big) \Big] \\
	&\qquad +
	\EEbothc \Big[ \ind\{\contextconc{}\}\sum_{t=1}^T \Big(\pseudoucb{t}(\actionrvc{t}) - \sum_{\postcontextval\in\postcontextspace}\EEenv{\actionrvc{t}}[\response{} \setdelim \postcontext{} = \postcontextval] \PPenv{\actionrvc{t}}[\postcontext{}=\postcontextval]\Big) \Big].
\]

We bound these terms separately. First, using the fact that $\priormargenv{}(\postcontext{})$ and $\env{}(\postcontext{})$ are $\eps$-close, the definition of $\contextconc{}$, and the definition of $\actionrvc{t}$,
\[\label{eqn:c-ucb-bound1}
	&\hspace{-1em}\EEbothc \Big[ \ind\{\contextconc{}\}\sum_{t=1}^T \Big(\sum_{\postcontextval\in\postcontextspace}\EEenv{\optactionval}[\response{} \setdelim \postcontext{} = \postcontextval] \PPenv{\optactionval}[\postcontext{}=\postcontextval] - \pseudoucb{t}(\actionrvc{t}) \Big) \Big] \\
	&\leq
	\EEbothc \Big[ \ind\{\contextconc{}\}\sum_{t=1}^T \Big(\sum_{\postcontextval\in\postcontextspace}\EEenv{\optactionval}[\response{} \setdelim \postcontext{} = \postcontextval] \PPpriorenv{\optactionval}[\postcontext{}=\postcontextval] - \pseudoucb{t}(\actionrvc{t}) \Big) \Big]
	+ \priormargscale T \\
	&\leq
	\EEbothc \Big[ \ind\{\contextconc{}\}\sum_{t=1}^T \Big(\pseudoucb{t}(\optactionval) - \pseudoucb{t}(\actionrvc{t}) \Big) \Big]
	+ \priormargscale T \\
	&\leq \priormargscale T.
\]
Second, using the definition of $\contextconc{}$, the fact that $\priormargenv{}(\postcontext{})$ and $\env{}(\postcontext{})$ are $\eps$-close, the \propertyname{} property, and \cref{fact:accumulated-contexts},
\[\label{eqn:c-ucb-bound2}
	&\hspace{-1em}\EEbothc \Big[ \ind\{\contextconc{}\}\sum_{t=1}^T \Big(\pseudoucb{t}(\actionrvc{t}) - \EEenv{\actionrvc{t}}[\response{} \,]\Big) \Big] \\
	&= 
	\EEbothc \Big[ \ind\{\contextconc{}\}\sum_{t=1}^T \Big(\sum_{\postcontextval\in\postcontextspace}\Big(\estcondmeancontext{t}(\postcontextval) + \sqrt{\log (2/\highprobparam{}) / (2\numobscontext{t}(\postcontextval))} \Big) \PPpriorenv{\actionrvc{t}}[\postcontext{}=\postcontextval] - \EEenv{\actionrvc{t}}[\response{} \,]\Big) \Big] \\ 
	&\leq
	\EEbothc \Big[ \ind\{\contextconc{}\}\sum_{t=1}^T \Big(\sum_{\postcontextval\in\postcontextspace}\Big(\EEenv{\actionrvc{t}}[\response{} \setdelim \postcontext{} = \postcontextval] + \sqrt{2\log (2/\highprobparam{}) / \numobscontext{t}(\postcontextval)} \Big) \PPpriorenv{\actionrvc{t}}[\postcontext{}=\postcontextval] - \EEenv{\actionrvc{t}}[\response{} \,]\Big) \Big] \\ 
	&\leq
	\EEbothc \Big[ \ind\{\contextconc{}\}\sum_{t=1}^T \Big(\sum_{\postcontextval\in\postcontextspace}\Big(\EEenv{\actionrvc{t}}[\response{} \setdelim \postcontext{} = \postcontextval] + \sqrt{2\log (2/\highprobparam{}) / \numobscontext{t}(\postcontextval)} \Big) \PPenv{\actionrvc{t}}[\postcontext{}=\postcontextval] - \EEenv{\actionrvc{t}}[\response{} \,]\Big) \Big]\\ 
	&\qquad + \priormargscale \Bigg(1 + \sqrt{2\log(2/\highprobparam{})}\Bigg) T  \\
	&=
	\EEbothc \Big[ \ind\{\contextconc{}\}\sum_{t=1}^T \sum_{\postcontextval\in\postcontextspace}\sqrt{2\log (2/\highprobparam{}) / \numobscontext{t}(\postcontextval)}  \PPenv{\actionrvc{t}}[\postcontext{}=\postcontextval]\Big] 
	+ \priormargscale \Bigg(1 + \sqrt{2\log(2/\highprobparam{})}\Bigg) T \\
	&\leq 4\sqrt{2\abs{\postcontextspace}T\log T}
	+ (\log T)\sqrt{2T}
	+ 4\sqrt{\log T}
	+ \priormargscale(1+2\sqrt{\log T})T.
\]
where the last line follows by taking $\highprobparam{} = 2/T^2$. 
The theorem then follows by combining \cref{eqn:c-ucb-bound0,eqn:c-ucb-bound1,eqn:c-ucb-bound2}.
\manualendproof

\subsection{Proof of \cref{fact:causal-failure}}\label{sec:proof-causal-failure}

We may assume without loss of generality that
\*[
	\lim_{T\to\infty} \frac{\log(1/\highprobparam{T})}{T} = 0.
\]
If this is not the case, then for large enough $T$, it holds that $\log(1/\highprobparam{T}) \geq c \, T$ for some $c$, and hence the instance dependent lower bounds for \UCB{} variants \citep{auer02nonstochastic} imply that \regretname{} grows linearly in $T$ in the worst-case. We may also assume  $\abs{\postcontextspace}>1$, for otherwise
 \CUCB{} plays the same arm forever (using its arbitrary tie-break rule) and so \CUCB{} can be forced to incur linear \regretname{} in a trivial way.  %

To illustrate that the lower bound is witnessed by a diversity of environments, we describe the construction in more general terms and then provide an example instantiation at the end of the proof. 
Let $\actionspaceidx{0}$ and $\postcontextspaceidx{0}$ be arbitrary, nonempty, strict subsets of $\actionspace$ and $\postcontextspace$ respectively, and let $\actionspaceidx{1}=\actionspace\setminus\actionspaceidx{0}$ and $\postcontextspaceidx{1}=\postcontextspace\setminus\postcontextspaceidx{0}$.

We now describe sufficient conditions for an \envname{} to be not \propertyname{} and to force \CUCB{} to incur linear \regretname{}.
For $\margmean(0)$ and $\margmean(1)$ in $(0,1)$, let the marginal distribution be such that for $\actionvalidx\in\{0,1\}$, if $\actionval\in\actionspaceidx{\actionvalidx}$ then
\*[
	\PP_{\env{\actionval}}[\postcontext{}=\postcontextval] = \frac{\margmean(\actionvalidx)}{\abs{\postcontextspaceidx{1}}}\ind\{\postcontextval\in\postcontextspaceidx{1}\} + \frac{1-\margmean(\actionvalidx)}{\abs{\postcontextspaceidx{0}}}\ind\{\postcontextval\in\postcontextspaceidx{0}\}.
\]
Similarly, for $\condmean(0,0)$, $\condmean(0,1)$, $\condmean(1,0)$, and $\condmean(1,1)$ in $(0,1)$, let the conditional distribution be such that for $\actionvalidx,\postcontextvalidx\in\{0,1\}$, if $\actionval\in\actionspaceidx{\actionvalidx}$ and $\postcontextval\in\postcontextspaceidx{\postcontextvalidx}$ then
\*[
	\env{\actionval}(\response{} \setdelim \postcontext{}=\postcontextval)
	= 
	\bernoullidist(\condmean(\actionvalidx,\postcontextvalidx)).
\]
Observe that for $\actionvalidx\in\{0,1\}$, if $\actionval\in\actionspaceidx{\actionvalidx}$ then
\*[
	\EE_{\env{\actionval}}[\response{}\,]
	= \condmean(\actionvalidx,1)\margmean(\actionvalidx) + \condmean(\actionvalidx,0)[1-\margmean(\actionvalidx)].
\]

We suppose the following conditions:
\begin{itemize}
\item[(1)] $\forall \actionval\in\actionspaceidx{0},\actionvaldum\in\actionspaceidx{1} \quad \EE_{\env{\actionval}}[\response{}\,] > \EE_{\env{\actionvaldum}}[\response{}\,],$
\item[(2)] $\margmean(0) < \margmean(1)$,
\item[(3)] $\min\{\condmean(0,1),  \condmean(1,1)\} > \max\{\condmean(0,0),  \condmean(1,0)\}$.
\end{itemize}

By Condition (1), $\optactionval\in\actionspaceidx{0}$ (note that then all $\actionval\in\actionspaceidx{0}$ are equally optimal), so a constant amount of \regretname{} is incurred whenever $\actionrvc{t}\in\actionspaceidx{1}$. We now argue that under Conditions (2) and (3), this happens for a constant fraction of rounds with high probability. 

First, we require slightly more notation to understand the behaviour of $\actionrvc{t}$. For every $t\in[T]$, $\actionvalidx\in\{0,1\}$, and $\postcontextval\in\postcontextspace$, let 
\*[
	\numobsboth{t}(\actionvalidx,\postcontextval)
	= \sum_{s=1}^t \ind\{\actionrvc{s}\in\actionspaceidx{\actionvalidx}, \postcontext{s}=\postcontextval\}.
\]
Define $\numobsaction{t}(\actionvalidx) = \sum_{\postcontextval\in\postcontextspace} \numobsboth{t}(\actionvalidx,\postcontextval)$, and note that $\numobscontext{t}(\postcontextval) = \numobsboth{t}(0,\postcontextval) + \numobsboth{t}(1,\postcontextval)$.
Further, let
\*[
	\estcondmeanboth{t}(\actionvalidx,\postcontextval) = \frac{1}{\numobsboth{t}(\actionvalidx,\postcontextval)} \sum_{s=1}^t \responserv{s}{\actionrvc{s}} \ind\{\actionrvc{s}\in\actionspaceidx{\actionvalidx}, \postcontext{s}=\postcontextval\},
\] 
By definition,
\*[
	\estcondmeancontext{t}(\postcontextval)
	= \frac{\numobsboth{t}(0,\postcontextval)}{\numobscontext{t}(\postcontextval)} \estcondmeanboth{t}(0,\postcontextval) + \frac{\numobsboth{t}(1,\postcontextval)}{\numobscontext{t}(\postcontextval)} \estcondmeanboth{t}(1,\postcontextval).
\]

Define the event $\lowerconc$ to be the case that for all $t\in[T]$, $\actionvalidx,\postcontextvalidx\in\{0,1\}$, and $\postcontextval\in\postcontextspaceidx{\postcontextvalidx}$,
\*[
	\abs[2]{\estcondmeanboth{t}(\actionvalidx,\postcontextval) - \condmean(\actionvalidx,\postcontextvalidx)} \leq \sqrt{\frac{\log T}{\numobsboth{t}(\actionvalidx,\postcontextval)}},
\]
and let $\lowerprobconc$ be the event that for all $t \in [T]$ and $\postcontextval\in\postcontextspace$,
\*[
	\frac{\numobscontext{t}(\postcontextval)}{t} 
	\geq \min_{\actionval\in\actionspace}\PP_{\env{\actionval}}[\postcontext{}=\postcontextval] - \sqrt{\frac{\log T}{t}}
\]
and
\*[
	\frac{\numobscontext{t}(\postcontextval)}{t} 
	\leq \max_{\actionval\in\actionspace}\PP_{\env{\actionval}}[\postcontext{}=\postcontextval] + \sqrt{\frac{\log T}{t}}\,.
\]
By Hoeffding and a union bound (i.e., the same arguments as \cref{fact:action-event-concentration,fact:context-event-concentration}), $\PPenv{}([\lowerconc\cap\lowerprobconc]\complement) \leq 8\abs{\postcontextspace}/T$.

Now, suppose the event $\lowerconc\cap\lowerprobconc$ holds, and consider a fixed $t$. 
Recall that $\actionrvc{t}\in\actionspaceidx{1}$ is implied by
\*[
	&\hspace{-1em}\max_{\actionval_1\in\actionspaceidx{1}}
	\Bigg\{\sum_{\postcontextval\in\postcontextspace}\Bigg(\estcondmeancontext{t}(\postcontextval) + \sqrt{\frac{\log(2/\highprobparam{T})}{2\numobscontext{t}(\postcontextval)}}\, \Bigg)\PP_{\env{\actionval_1}}[\postcontext{}=\postcontextval] \Bigg\} \\
	&>
	\max_{\actionval_0\in\actionspaceidx{0}}
	\Bigg\{\sum_{\postcontextval\in\postcontextspace}\Bigg(\estcondmeancontext{t}(\postcontextval) + \sqrt{\frac{\log(2/\highprobparam{T})}{2\numobscontext{t}(\postcontextval)}}\, \Bigg)\PP_{\env{\actionval_0}}[\postcontext{}=\postcontextval] \Bigg\}.
\]
This is equivalent to
\*[
	0 
	&<
	\sum_{\postcontextval\in\postcontextspaceidx{1}}
	\Bigg(
	\frac{\numobsboth{t}(0,\postcontextval)}{\numobscontext{t}(\postcontextval)} \estcondmeanboth{t}(0,\postcontextval) + \frac{\numobsboth{t}(1,\postcontextval)}{\numobscontext{t}(\postcontextval)} \estcondmeanboth{t}(1,\postcontextval) + \sqrt{\frac{\log(2/\highprobparam{})}{2\numobscontext{t}(\postcontextval)}}\, 
	\Bigg)
	\frac{\margmean(1)-\margmean(0)}{\abs{\postcontextspaceidx{1}}} \\
	&\qquad -	
	\sum_{\postcontextval\in\postcontextspaceidx{0}}
	\Bigg(
	\frac{\numobsboth{t}(0,\postcontextval)}{\numobscontext{t}(\postcontextval)} \estcondmeanboth{t}(0,\postcontextval) + \frac{\numobsboth{t}(1,\postcontextval)}{\numobscontext{t}(\postcontextval)} \estcondmeanboth{t}(1,\postcontextval) + \sqrt{\frac{\log(2/\highprobparam{})}{2\numobscontext{t}(\postcontextval)}}\, 
	\Bigg)
	\frac{\margmean(1)-\margmean(0)}{\abs{\postcontextspaceidx{0}}}.
\]

The following two cases hold.
\begin{itemize}
\item[(i)] For all $\postcontextval\in\postcontextspaceidx{1}$,
\*[
	&\hspace{-1em}
	\frac{\numobsboth{t}(0,\postcontextval)}{\numobscontext{t}(\postcontextval)} \estcondmeanboth{t}(0,\postcontextval) + \frac{\numobsboth{t}(1,\postcontextval)}{\numobscontext{t}(\postcontextval)} \estcondmeanboth{t}(1,\postcontextval) + \sqrt{\frac{\log(2/\highprobparam{})}{2\numobscontext{t}(\postcontextval)}} \\
	&\geq  
	\frac{\numobsboth{t}(0,\postcontextval)}{\numobscontext{t}(\postcontextval)} \Bigg(\condmean(0,1) - \sqrt{\frac{\log T}{\numobsboth{t}(0,\postcontextval)}}\, \Bigg) + \frac{\numobsboth{t}(1,\postcontextval)}{\numobscontext{t}(\postcontextval)} \Bigg(\condmean(1,1) - \sqrt{\frac{\log T}{\numobsboth{t}(1,\postcontextval)}}\, \Bigg) \\
	&\geq  
	\min\{\condmean(0,1),  \condmean(1,1)\}
	- 2\sqrt{\log T} \Big(t\, \margmean(0)/\abs{\postcontextspaceidx{1}} - \sqrt{t\log T}\Big)^{-1/2}.
\]
\item[(ii)] For all $\postcontextval\in\postcontextspaceidx{0}$,
\*[
	&\hspace{-1em}
	\frac{\numobsboth{t}(0,\postcontextval)}{\numobscontext{t}(\postcontextval)} \estcondmeanboth{t}(0,\postcontextval) + \frac{\numobsboth{t}(1,\postcontextval)}{\numobscontext{t}(\postcontextval)} \estcondmeanboth{t}(1,\postcontextval) + \sqrt{\frac{\log(2/\highprobparam{})}{2\numobscontext{t}(\postcontextval)}} \\
	&\leq  
	\frac{\numobsboth{t}(0,\postcontextval)}{\numobscontext{t}(\postcontextval)} \Bigg(\condmean(0,0) + \sqrt{\frac{\log T}{\numobsboth{t}(0,\postcontextval)}}\, \Bigg) + \frac{\numobsboth{t}(1,\postcontextval)}{\numobscontext{t}(\postcontextval)} \Bigg(\condmean(1,0) + \sqrt{\frac{\log T}{\numobsboth{t}(1,\postcontextval)}}\, \Bigg) + \sqrt{\frac{\log(2/\highprobparam{})}{2\numobscontext{t}(\postcontextval)}}\\
	&\leq  
	\max\{\condmean(0,0),  \condmean(1,0)\}
	+ 2\sqrt{\log T} \Big(t (1-\margmean(1))/\abs{\postcontextspaceidx{0}} - \sqrt{t\log T}\Big)^{-1/2} \\ 
	&\qquad + \sqrt{\log(2/\highprobparam{})} \Big(t (1-\margmean(1))/\abs{\postcontextspaceidx{0}} - \sqrt{t\log T}\Big)^{-1/2}.
\]
\end{itemize}

That is,
\*[
	&\hspace{-1em}
	\sum_{\postcontextval\in\postcontextspaceidx{1}}
	\Bigg(
	\frac{\numobsboth{t}(0,\postcontextval)}{\numobscontext{t}(\postcontextval)} \estcondmeanboth{t}(0,\postcontextval) + \frac{\numobsboth{t}(1,\postcontextval)}{\numobscontext{t}(\postcontextval)} \estcondmeanboth{t}(1,\postcontextval) + \sqrt{\frac{\log(2/\highprobparam{})}{2\numobscontext{t}(\postcontextval)}}\, 
	\Bigg)
	\frac{\margmean(1)-\margmean(0)}{\abs{\postcontextspaceidx{1}}} \\
	&\qquad -	
	\sum_{\postcontextval\in\postcontextspaceidx{0}}
	\Bigg(
	\frac{\numobsboth{t}(0,\postcontextval)}{\numobscontext{t}(\postcontextval)} \estcondmeanboth{t}(0,\postcontextval) + \frac{\numobsboth{t}(1,\postcontextval)}{\numobscontext{t}(\postcontextval)} \estcondmeanboth{t}(1,\postcontextval) + \sqrt{\frac{\log(2/\highprobparam{})}{2\numobscontext{t}(\postcontextval)}}\, 
	\Bigg)
	\frac{\margmean(1)-\margmean(0)}{\abs{\postcontextspaceidx{0}}} \\
	&\geq
	\Big[ \min\{\condmean(0,1),  \condmean(1,1)\}
	- 2\sqrt{\log T} \Big(t\, \margmean(0)/\abs{\postcontextspaceidx{1}} - \sqrt{t\log T}\Big)^{-1/2} \\
	&\qquad - \max\{\condmean(0,0),  \condmean(1,0)\}
	- 2\sqrt{\log T} \Big(t (1-\margmean(1))/\abs{\postcontextspaceidx{0}} - \sqrt{t\log T}\Big)^{-1/2} \\
	&\qquad - \sqrt{\log(2/\highprobparam{})} \Big(t (1-\margmean(1))/\abs{\postcontextspaceidx{0}} - \sqrt{t\log T}\Big)^{-1/2}\Big]\Big(\margmean(1)-\margmean(0) \Big).
\]
By Condition (3), for large enough $T$ and $t \geq T/2$, this last step can be further lower bounded\footnote{When $\highprobparam{T}$ is polynomial in $T$, this can be improved to only require $t \geq (\log T)^2$.} using
\*[
	&\hspace{-1em} \min\{\condmean(0,1),  \condmean(1,1)\}
	- 2\sqrt{\log T} \Big(t\, \margmean(0)/\abs{\postcontextspaceidx{1}} - \sqrt{t\log T}\Big)^{-1/2} \\
	&\qquad - \max\{\condmean(0,0),  \condmean(1,0)\}
	- 2\sqrt{\log T} \Big(t (1-\margmean(1))/\abs{\postcontextspaceidx{0}} - \sqrt{t\log T}\Big)^{-1/2} \\
	&\qquad - \sqrt{\log(2/\highprobparam{})} \Big(t (1-\margmean(1))/\abs{\postcontextspaceidx{0}} - \sqrt{t\log T}\Big)^{-1/2}\\
	&\geq \Big(\min\{\condmean(0,1),  \condmean(1,1)\} - \max\{\condmean(0,0),  \condmean(1,0)\}\Big)/2.
\]

Thus, we have that for sufficiently large $T$ and any $\actionval_{0}\in\actionspaceidx{0}$ and $\actionval_{1}\in\actionspaceidx{1}$
\*[
	\regretc{T}
	&=\EEbothc \sum_{t=1}^T (\EE_{\env{\actionval_{0}}}[\response{}\,] - \EE_{\env{\actionrvc{t}}}[\response{}\,])  \\
	&\geq \EEbothc \Big[\ind\{\lowerconc \cap \lowerprobconc\} \sum_{t=1}^T (\EE_{\env{\actionval_{0}}}[\response{}\,] - \EE_{\env{\actionrvc{t}}}[\response{}\,]) \Big] \\
	&\geq (T/4) \Big(\EE_{\env{\actionval_{0}}}[\response{}\,] - \EE_{\env{\actionval_{1}}}[\response{}\,]\Big) (1-8\abs{\postcontextspace}/T) \\
	&\geq (T/5) \Big(\EE_{\env{\actionval_{0}}}[\response{}\,] - \EE_{\env{\actionval_{1}}}[\response{}\,]\Big).
\] 

Finally, a concrete example of an \envname{} that satisfies Conditions (1)--(3) is:
\*[
	&\condmean(0,0) = 1/6 \quad
	&&\condmean(1,0) = 2/6 \\
	&\condmean(0,1) = 5/6 \quad
	&&\condmean(1,1) = 4/6\\
	&\margmean(0) = 6/8 \quad
	&&\margmean(1) = 7/8.
\]
For simplicity, we use the constants from this example in the theorem statement.
\manualendproof%

\subsection{Proof of \cref{fact:improved-causal}}\label{sec:proof-improved-causal}

First, by \cref{fact:action-event-concentration,fact:marginal-event-concentration} with $\highprobparam{} = 2/T^2$,
\[\label{eqn:hyptest-cb-0}
	\regrethc{T}
	&=\EEbothhc \sum_{t=1}^T \Big[\EEenv{\optactionval}[\response{} \, ] - \EEenv{\actionrvhc{t}}[\response{}\,] \Big] \\
	&\leq 4\abs{\actionspace} + \EEbothhc \Big[ \ind\{\actionconc{} \cap \margconc\}\sum_{t=1}^T \Big(\EEenv{\optactionval}[\response{} \, ] - \EEenv{\actionrvhc{t}}[\response{}\,]\Big) \Big].
\]

Let $\exploretime$ be the last round on which the algorithm uniformly explores and $\switchtime$ be the last round on which $\actionrvhc{t} = \actionrvc{t}$. Since the algorithm is deterministic, $\exploretime$ is fixed but $\switchtime$ is stochastic. Then,
\[\label{eqn:hyptest-epoch-decomp}
	&\hspace{-1em}\EEbothhc \Bigg[\ind\{\actionconc{}\cap\margconc\}\sum_{t=1}^T\Big(\EEenv{\optactionval}[\response{} \, ] - \EEenv{\actionrvhc{t}}[\response{}\,] \Big) \Bigg] \\
	&= \exploretime
	+ \EEbothhc \Bigg[\ind\{\actionconc{}\cap\margconc\} \sum_{\exploretime < t \leq\switchtime} \Big(\EEenv{\optactionval}[\response{} \, ] - \EEenv{\actionrvc{t}}[\response{}\,] \Big)\Bigg]\\
	&\hspace{1em}+ \EEbothhc \Bigg[\ind\{\actionconc{}\cap\margconc\} \sum_{t>\switchtime} \Big(\EEenv{\optactionval}[\response{} \, ] - \EEenv{\actionrvucb{t}}[\response{}\,] \Big) \Bigg].
\]

Suppose that $\env{}$ is \propertyname{} and $\priormargenv{}(\postcontext{})$ and $\env{}(\postcontext{})$ are $\priormargscale$-close for $\priormargscale \leq T^{-1/4} \sqrt{\abs{\actionspace}\abs{\postcontextspace}\log T}$. By triangle inequality, on the event $\margconc$
\*[
	\sup_{\actionval\in\actionspace} \sum_{\postcontextval\in\postcontextspace} \abs[2]{\PPpriorenv{\actionval}[\postcontext{} = \postcontextval] - \PPestenv{\actionval}[\postcontext{} = \postcontextval]}
	&\leq 
	\frac{2\sqrt{\abs{\actionspace}\abs{\postcontextspace}\log T}}{T^{1/4}},
\]
and thus $\priormargenv{}(\postcontext{})$ will \emph{not} be replaced by $\estmargenv{}(\postcontext{})$. 
Further, on $\actionconc{} \cap \contextconc{}$ with $\highprobparam{}=2/T^2$, for all $t$ it holds that
\*[
	&\hspace{-1em}\ucbaction{t-1}(\actionval) - \pseudoucb{t-1}(\actionval) \\
	&\geq \EEenv{\actionval}[\response{} \,] - \sum_{\postcontextval\in\postcontextspace} \Bigg(\EEenv{}[\response{} \setdelim \postcontext{} = \postcontextval] + 2\sqrt{\frac{\log T}{\numobscontext{t-1}(\postcontextval)}} \, \Bigg) \PPpriorenv{\actionval}[\postcontext{}=\postcontextval] \\
	&= \sum_{\postcontextval\in\postcontextspace} \EEenv{}[\response{} \setdelim \postcontext{}=\postcontextval] \PPenv{\actionval}[\postcontext{}=\postcontextval] - \sum_{\postcontextval\in\postcontextspace} \Bigg(\EEenv{}[\response{} \setdelim \postcontext{} = \postcontextval] + 2\sqrt{\frac{\log T}{\numobscontext{t-1}(\postcontextval)}} \, \Bigg) \PPpriorenv{\actionval}[\postcontext{}=\postcontextval] \\
	&\geq - 2\sum_{\postcontextval\in\postcontextspace} \sqrt{\frac{\log T}{\numobscontext{t-1}(\postcontextval)}}  \PPpriorenv{\actionval}[\postcontext{}=\postcontextval] - \frac{\sqrt{\abs{\actionspace}\abs{\postcontextspace}\log T}}{T^{1/4}}
\]
and
\*[
	&\hspace{-1em}\ucbaction{t-1}(\actionval) - \pseudoucb{t-1}(\actionval) \\
	&\leq \EEenv{\actionval}[\response{} \,] + 2\sqrt{\frac{\log T}{\numobsaction{t-1}(\actionval)}}
	- \sum_{\postcontextval\in\postcontextspace} \EEenv{}[\response{} \setdelim \postcontext{}=\postcontextval] \PPpriorenv{\actionval}[\postcontext{}=\postcontextval] \\
	&= \sum_{\postcontextval\in\postcontextspace} \EEenv{}[\response{} \setdelim \postcontext{}=\postcontextval] \PPenv{\actionval}[\postcontext{}=\postcontextval] + 2\sqrt{\frac{\log T}{\numobsaction{t-1}(\actionval)}}
	- \sum_{\postcontextval\in\postcontextspace} \EEenv{}[\response{} \setdelim \postcontext{}=\postcontextval] \PPpriorenv{\actionval}[\postcontext{}=\postcontextval] \\
	&\leq 2\sqrt{\frac{\log T}{\numobsaction{t-1}(\actionval)}} + \frac{\sqrt{\abs{\actionspace}\abs{\postcontextspace}\log T}}{T^{1/4}},
\]
and hence $\switchtime = T$.
Thus, by \cref{fact:context-event-concentration} with $\highprobparam{}=2/T^2$, we can actually bound \cref{eqn:hyptest-epoch-decomp} using
\*[	
	&\hspace{-1em}\EEbothhc \Bigg[\ind\{\actionconc{}\cap\margconc\}\sum_{t=1}^T\Big(\EEenv{\optactionval}[\response{} \, ] - \EEenv{\actionrvhc{t}}[\response{}\,] \Big) \Bigg] \\
	&\leq\exploretime + 2\abs{\postcontextspace} + \EEbothhc \Big[ \ind\{\actionconc{}\cap\contextconc{}\cap\margconc\}\sum_{t=\exploretime+1}^T \Big(\EEenv{\optactionval}[\response{} \, ] - \EEenv{\actionrvhc{t}}[\response{}\,]\Big) \Big] \\
	&\leq \exploretime + 2\abs{\postcontextspace} + \EEbothc \Big[ \ind\{\actionconc{}\cap\contextconc{}\}\sum_{t=1}^T \Big(\EEenv{\optactionval}[\response{} \, ] - \EEenv{\actionrvc{t}}[\response{}\,]\Big) \Big] \\
	&\leq 5\sqrt{T} + 2\abs{\postcontextspace}
	+ 4\sqrt{2\abs{\postcontextspace}T\log T}
	+ (\log T)\sqrt{2T}
	+ 4\sqrt{\log T}
	+ \priormargscale(2+2\sqrt{\log T})T,
\]
where the last line follows from \cref{eqn:c-ucb-bound1,eqn:c-ucb-bound2} and the definition of $\exploretime$. The statement follows from combining this with \cref{eqn:hyptest-cb-0}.

Otherwise, consider when $\env{}$ is not \propertyname{}.
By \cref{fact:existing-ucb},
\[\label{eqn:hyptest-worst1}
	&\hspace{-1em}\EEbothhc \Bigg[\ind\{\actionconc{}\cap\margconc\} \sum_{t>\switchtime} \Big(\EEenv{\optactionval}[\response{}] - \EEenv{\actionrvhc{t}}[\response{}] \Big) \Bigg] \\
	&\leq \EEbothucb \Bigg[\ind\{\actionconc{}\} \sum_{t=1}^T \Big(\EEenv{\optactionval}[\response{} \, ] - \EEenv{\actionrvucb{t}}[\response{}\,] \Big) \Bigg] \\
	&\leq 4\sqrt{2\abs{\actionspace}T \log T}.
\]

It remains to focus on the \regretname{} contribution from when $\actionrvhc{t}=\actionrvc{t}$. The key observation is that if $t \leq \switchtime$ then the hypothesis test passed for this round. We decompose the regret incurred using
\*[
	&\hspace{-1em}\EEbothhc \Bigg[\ind\{\actionconc{}\cap\margconc\} \sum_{\exploretime < t \leq\switchtime} \Big(\EEenv{\optactionval}[\response{} \, ] - \EEenv{\actionrvc{t}}[\response{}\,] \Big)\Bigg]\\
	&=\EEbothhc \Bigg[\ind\{\actionconc{}\cap\margconc\} \sum_{\exploretime < t \leq\switchtime} \Big(\EEenv{\optactionval}[\response{} \, ]- \pseudoucb{t-1}(\actionrvc{t}) + \pseudoucb{t-1}(\actionrvc{t}) - \EEenv{\actionrvc{t}}[\response{} \,]\Big) \Bigg].
\]

First, on the event $\actionconc{}\cap\margconc$ with $\highprobparam{}=2/T^2$,
\[\label{eqn:hyptest-worst2}
	&\hspace{-1em}\sum_{\exploretime < t \leq\switchtime} \Big(\EEenv{\optactionval}[\response{} \, ]- \pseudoucb{t-1}(\actionrvc{t})\Big) \\
	&\leq_{(a)} \sum_{\exploretime < t \leq\switchtime} \Big(\ucbaction{t-1}(\optactionval)- \pseudoucb{t-1}(\actionrvc{t})\Big) \\
	&\leq_{(b)} \sum_{\exploretime < t \leq\switchtime} \Bigg(\pseudoucb{t-1}(\optactionval)- \pseudoucb{t-1}(\actionrvc{t}) + 2\sqrt{\frac{\log T}{\numobsaction{t-1}(\optactionval)}}\,\Bigg) + 2T^{3/4}\sqrt{\abs{\actionspace}\abs{\postcontextspace}\log T} \\
	&\leq_{(c)} 2\sum_{\exploretime < t \leq\switchtime} \sqrt{\frac{\log T}{\numobsaction{t-1}(\optactionval)}} + 2T^{3/4}\sqrt{\abs{\actionspace}\abs{\postcontextspace}\log T}\\
	&\leq_{(d)} 2\sum_{\exploretime < t \leq\switchtime} \sqrt{\frac{\abs{\actionspace} \log T}{\sqrt{T}}} + 2T^{3/4}\sqrt{\abs{\actionspace}\abs{\postcontextspace}\log T} \\
	&\leq 4 \,T^{3/4}\sqrt{\abs{\actionspace}\abs{\postcontextspace}\log T},
\]
where we have used that
(a) $\actionconc{}$ holds;
(b) the hypothesis test for \HCUCB{} passed;
(c) $\actionrvc{t} = \argmax_{\actionval\in\actionspace}\pseudoucb{t-1}(\actionval)$; and
(d) for $\exploretime < t$, $\numobsaction{t-1}(\optactionval) \geq \numobsaction{\exploretime}(\optactionval) \geq \sqrt{T} / \abs{\actionspace}$.

Second, again on the event $\actionconc{}\cap\margconc$ with $\highprobparam{}=2/T^2$,
\[\label{eqn:hyptest-worst3}
	&\hspace{-1em}\sum_{\exploretime < t \leq\switchtime} \Big(\pseudoucb{t-1}(\actionrvc{t}) - \EEenv{\actionrvc{t}}[\response{} \,]\Big)\\
	&\leq_{(a)} \sum_{\exploretime < t \leq\switchtime} \Bigg(\ucbaction{t-1}(\actionrvc{t}) + 2\sum_{\postcontextval\in\postcontextspace} \sqrt{\frac{\log T}{\numobscontext{t-1}(\postcontextval)}}\, \PPpriorenv{\actionval}[\postcontext{}{}=\postcontextval] + \frac{2\sqrt{\abs{\actionspace}\abs{\postcontextspace}\log T}}{T^{1/4}} - \EEenv{\actionrvc{t}}[\response{} \,]\Bigg) \\
	&\leq_{(b)} \sum_{\exploretime < t \leq\switchtime}  \Bigg(2\sqrt{\frac{\log T}{\numobsaction{t-1}(\actionrvhc{t})}} + 2\sum_{\postcontextval\in\postcontextspace} \sqrt{\frac{\log T}{\numobscontext{t-1}(\postcontextval)}}\, \PPpriorenv{\actionval}[\postcontext{}{}=\postcontextval] + \frac{\sqrt{\abs{\actionspace}\abs{\postcontextspace}\log T}}{T^{1/4}}\Bigg) \\
	&\leq_{(c)} \sum_{\exploretime < t \leq\switchtime}  \Bigg(2\sqrt{\frac{\log T}{\numobsaction{t-1}(\actionrvhc{t})}} + 2\sum_{\postcontextval\in\postcontextspace} \sqrt{\frac{\log T}{\numobscontext{t-1}(\postcontextval)}}\, \PPenv{\actionval}[\postcontext{}{}=\postcontextval] + 7\frac{(\log T)\sqrt{\abs{\actionspace}\abs{\postcontextspace}}}{T^{1/4}}\Bigg) \\
	&\leq_{(d)} 4\sqrt{2\abs{\actionspace}T\log T} + 7 T^{3/4} (\log T) \sqrt{\abs{\actionspace}\abs{\postcontextspace}} + 2\sum_{\exploretime < t \leq\switchtime} \sum_{\postcontextval\in\postcontextspace} \sqrt{\frac{\log T}{\numobscontext{t-1}(\postcontextval)}}\, \PPenv{\actionval}[\postcontext{}{}=\postcontextval],
\]
where we have used that
(a) the hypothesis test for \HCUCB{} passed;
(b) $\actionconc{}$ holds;
(c) $\priormargenv{}(\postcontext{})$ and $\env{}(\postcontext{})$ are $3T^{-1/4} \sqrt{\abs{\actionspace}\abs{\postcontextspace}\log T}$-close (if the original $\priormargenv{}$ does not satisfy this, then by triangle inequality it is replaced with $\estmargenv{}$ that does satisfy this); and
(d) \cref{fact:accumulated-actions}.
It remains to take expectation and apply \cref{fact:accumulated-contexts}, and then combine \cref{eqn:hyptest-worst1,eqn:hyptest-worst2,eqn:hyptest-worst3} to obtain
\[\label{eqn:hyptest-constants}
	\regrethc{T}
	&\leq 4\abs{\actionspace} + 4\sqrt{2\abs{\actionspace}T \log T} + 5 \sqrt{T}
	+ 4 \,T^{3/4}\sqrt{\abs{\actionspace}\abs{\postcontextspace}\log T} \\
	&\qquad + 4\sqrt{2\abs{\actionspace}T\log T} + 7 T^{3/4} (\log T) \sqrt{\abs{\actionspace}\abs{\postcontextspace}} \\
	&\qquad+ 2\sqrt{\log T}\Big[\sqrt{8\abs{\postcontextspace} T} + \sqrt{(T/2)\log T} + 2 \Big].
\]
\manualendproof

\subsection{Proof of \cref{fact:improved-causal-special}}\label{sec:proof-improved-causal-special}

Supposing the event $\actionconc{}$ holds (from \cref{fact:action-event-concentration}), every $\actionval\in\actionspaceidx{0}$ satisfies
\*[
	\ucbaction{t-1}(\actionval) \geq \condmean(0,1)\margmean(0) + \condmean(0,0)(1-\margmean(0)).
\] 
We use the same \envname{} construction from the proof of \cref{fact:causal-failure} in \cref{sec:proof-causal-failure}.
Using the specific choice of $\highprobparam{T}$, this argument implies that on the event $\lowerconc\cap\lowerprobconc$, $\actionrvc{t} \in \actionspaceidx{1}$ for sufficiently large $T$ and $t \geq 2\sqrt{T}$. 
Recall that
\*[
	\estcondmeancontext{t}(\postcontextval) 
	&=
	\frac{\numobsboth{t}(0,\postcontextval)}{\numobscontext{t}(\postcontextval)} \estcondmeanboth{t}(0,\postcontextval) + \frac{\numobsboth{t}(1,\postcontextval)}{\numobscontext{t}(\postcontextval)} \estcondmeanboth{t}(1,\postcontextval).
\]
Since $\numobsboth{t}(0,\postcontextval) \leq 2\sqrt{T}$ and $\numobscontext{t}(\postcontextval)$ grows linearly in $t$ on the event $\lowerconc\cap\lowerprobconc$, for every $\alpha\in(0,1)$ and $\eps>0$ there exists $T$ large enough such that for all $t \geq (\log T)\sqrt{T}$ (crucially, the $\log T$ ensures that the proportion of these $t$ where $\actionrvc{t} \in \actionspaceidx{1}$ tends to 1 as $T$ gets larger) it holds that for $\postcontextval\in\postcontextspaceidx{\postcontextvalidx}$,
\*[
	\estcondmeancontext{t}(\postcontextval) 
	\leq \alpha \condmean(0,\postcontextvalidx) + (1-\alpha) \condmean(1,\postcontextvalidx) + \eps.
\]
This implies that for all $\eps>0$, taking $\alpha$ small enough and $T$ large enough with $t \geq (\log T)\sqrt{T}$ gives
\*[
	\pseudoucb{t-1}(\actionval)
	\leq
	\condmean(1,1)\margmean(0) + \condmean(1,0)(1-\margmean(0)) + \eps.
\]
By the exploration phase,
\*[
	2\sqrt{\frac{\log T}{\numobsaction{t-1}(\actionval)}}
	\leq 2 \sqrt{\frac{\abs{\actionspace} \log T}{\sqrt{T}}},
\]
and hence can be made arbitrarily small by taking $T$ sufficiently large. Thus, under the assumption
\[\label{eqn:specific-condition}
	[\condmean(0,1)-\condmean(1,1)]\margmean(0) + [\condmean(0,0)-\condmean(1,0)](1-\margmean(0)) > 0,
\]
for large enough $T$ the first condition from \cref{alg:hyp-test} will fail for some $\actionval\in\actionspaceidx{0}$ when $t = (\log T)\sqrt{T}$. Note that \cref{eqn:specific-condition} is satisfied by the example given in the proof of \cref{fact:causal-failure}. This implies that, on the event $\lowerconc\cap\lowerprobconc\cap\actionconc{}$, for large enough $T$ the \HCUCB{} algorithm switches to following \UCB{} when $t = (\log T)\sqrt{T}$. Since this joint event holds with probability larger than $1 - 2(\abs{\actionspace}+\abs{\postcontextspace})/T$, combining the exploration phase regret with the regret bound of \cref{fact:existing-ucb} gives the result.
\manualendproof

%% file: section-files/main-sections/adaptive-proof.tex
\section{Proof of \cref{fact:adaptive-impossible-short}}\label{sec:adaptive-proofs}

Fix $\actionspace$, $\postcontextspace$, and $T$.
Let $\postcontextspaceidx{0}$ be an arbitrary strict subset of $\postcontextspace$ and $\postcontextspaceidx{1}=\postcontextspace\setminus\postcontextspaceidx{0}$.
Fix $\Delta\in(0,1/20)$ and $\eps\in(0,1)$ to be chosen later.
Define the family of marginal distributions
\*[
	\margdist_{\actionval}[\postcontext{}\in\postcontextspaceidx{0}]
	=
	\begin{cases}
		1/2 + 2\Delta
		&\actionval=1 \\
		1/2
		&\actionval\neq1,
	\end{cases}
\]
where probability is evenly spaced within $\postcontextspaceidx{0}$ and $\postcontextspaceidx{1}$ respectively.
Further, define the Bernoulli conditional response distribution
\*[
	\baseconddist[\response{}=1 \setdelim \postcontext{}=\postcontextval]
	= 
	\begin{cases}
	3/4
	&\postcontextval\in\postcontextspaceidx{0}\\
	1/4
	&\postcontextval\in\postcontextspaceidx{1}.
	\end{cases}
\]

Define the Bernoulli \propertyname{} \envname{} $\baseenv{}\in\envspace$
for all $\actionval\in\actionspace$ by 
\*[
	\PP_{\baseenv{\actionval}}[\response{}=1] = \sum_{\postcontextval\in\postcontextspace} \baseconddist[\response{}=1 \setdelim \postcontext{}=\postcontextval] \margdist_\actionval[\postcontext{}=\postcontextval].
\]
Notice that
\*[
	\EE_{\baseenv{\actionval}}[\response{}\,]
	=
	\begin{cases}
	1/2 + \Delta
	&\actionval=1 \\
	1/2
	&\actionval\neq1.
	\end{cases}
\]

Then, for every $\refactionval\neq1$, define
$\specenv{}{\Delta}{\refactionval}\in\envspace$ for all $\actionval\in\actionspace$ by
\*[
	\PP_{\specenv{\actionval}{\Delta}{\refactionval}}[\response{}=1] = \sum_{\postcontextval\in\postcontextspace} \conddist{\refactionval}_\actionval[\response{}=1 \setdelim \postcontext{}=\postcontextval] \margdist_\actionval[\postcontext{}=\postcontextval],
\]
where
\*[
	\conddist{\refactionval}_\actionval[\response{}=1 \setdelim \postcontext{}=\postcontextval] 
	=
	\begin{cases}
	3/4
	&\actionval=1,\postcontextval\in\postcontextspaceidx{0}\\
	1/4
	&\actionval=1,\postcontextval\in\postcontextspaceidx{1}\\
	3/4 + 2\Delta(1+\eps)
	&\actionval=\refactionval,\postcontextval\in\postcontextspaceidx{0}\\
	1/4
	&\actionval=\refactionval,\postcontextval\in\postcontextspaceidx{1}\\
	3/4
	&\actionval\not\in\{1,\refactionval\}, \postcontextval\in\postcontextspaceidx{0}\\
	1/4
	&\actionval\not\in\{1,\refactionval\}, \postcontextval\in\postcontextspaceidx{1}.
	\end{cases}
\]
Notice that $\specenv{}{\Delta}{\refactionval}$ is \emph{not} \propertyname{}, and
\*[
	\EE_{\specenv{\actionval}{\Delta}{\refactionval}}[\response{}\,]
	=
	\begin{cases}
		1/2 + \Delta
		&\actionval=1 \\
		1/2 + \Delta(1+\eps)
		&\actionval=\refactionval \\
		1/2
		&\actionval\not\in\{1,\refactionval\}.
	\end{cases}
\]

We now extend Lemma 15.1 of \citet{bandit20book}.
In particular, let $\policymarg{} = \algo(\actionspace,\postcontextspace,\margdist,T)$ and observe that for any $\env{}\in\envmargspace(\margdist)$, 
\*[
	\dee \PP_{\env{},\policymarg{}}(\historyrv{T})
	&= \prod_{t=1}^T \policymarg{t}(\actionrv{t}\setdelim\historyrv{t-1})\PPenv{}[\postcontextrv{t}{\actionrv{t}},\responserv{t}{\actionrv{t}}\setdelim\actionrv{t}].
\]
Thus, for every $\refactionval\neq 1$,
\*[
	\KL{\PP_{\baseenv{},\policymarg{}}}{\PP_{\specenv{}{\Delta}{\refactionval}, \policymarg{}}}
	&= \EE_{\baseenv{},\policymarg{}}\Big[\log\frac{\dee \PP_{\baseenv{},\policymarg{}}}{\dee \PP_{\specenv{}{\Delta}{\refactionval}, \policymarg{}}}(\historyrv{T})\Big] \\
	&= \sum_{t=1}^T \EE_{\baseenv{},\policymarg{}}\Bigg[\EE_{\baseenv{},\policymarg{}}\Bigg[ \log \frac{\PP_{\baseenv{}}[\postcontextrv{t}{\actionrv{t}},\responserv{t}{\actionrv{t}}]}{\PP_{\specenv{}{\Delta}{\refactionval}}[\postcontextrv{t}{\actionrv{t}},\responserv{t}{\actionrv{t}}]}\Bigsetdelim\actionrv{t} \Bigg]\Bigg] \\
	&= \sum_{t=1}^T \EE_{\baseenv{},\policymarg{}} \Bigg[\KL{\PP_{\baseenv{\actionrv{t}}}}{\PP_{\specenv{\actionrv{t}}{\Delta}{\refactionval}}} \Bigg] \\
	&= \sum_{\actionval\in\actionspace} \EE_{\baseenv{},\policymarg{}}[\numobsaction{T}(\actionval)] \,\KL{\PP_{\baseenv{\actionval}}}{\PP_{\specenv{\actionval}{\Delta}{\refactionval}}}.
\]

Since the marginal distribution $\margdist$ is shared, 
for each $\actionval\in\actionspace$ this simplifies to
\*[
	\KL{\PP_{\baseenv{\actionval}}}{\PP_{\specenv{\actionval}{\Delta}{\refactionval}}}
	&= \sum_{\postcontextval\in\postcontextspace} \margdist_\actionval(\postcontext{} = \postcontextval) \KL{\baseenv{\actionval}(\response{} \setdelim \postcontextval)}{\specenv{\actionval}{\Delta}{\refactionval}(\response{} \setdelim \postcontextval)} \\
	&=
	\begin{cases}
	0
	& \actionval\neq\refactionval \\
	(1/2)\KL{\bernoullidist(3/4)}{\bernoullidist(3/4+2\Delta(1+\eps))}
	& \actionval=\refactionval.
	\end{cases}
\]

Thus, by Pinsker's inequality (Theorem~14.2 of \citep{bandit20book}),
\*[
	\regretcustom{\baseenv{}}{\policymarg{}}{T} + \regretcustom{\specenv{}{\Delta}{\refactionval}}{\policymarg{}}{T}
	&> \frac{T\Delta}{2}\PP_{\baseenv{},\policymarg{}}[\numobsaction{T}(1) \leq T/2] 
	+ \frac{T\Delta\eps}{2} \PP_{\specenv{}{\Delta}{\refactionval},\policymarg{}}[\numobsaction{T}(1) > T/2] \\
	&\geq
	\frac{T\Delta\eps}{4}\exp\{-\KL{\PP_{\baseenv{},\policymarg{}}}{\PP_{\specenv{}{\Delta}{\refactionval}, \policymarg{}}}\} \\
	&= \frac{T\Delta\eps}{4}\exp\Big\{-\EE_{\baseenv{},\policymarg{}}[\numobsaction{T}(\refactionval)] \,(1/2)\KL{\bernoullidist(3/4)}{\bernoullidist(3/4+2\Delta(1+\eps))}\Big\}.
\]
Using that $\KL{\bernoullidist(3/4)}{\bernoullidist(3/4+2\Delta(1+\eps))} \leq 4x^2$ for $x<1/10$ and the assumption of the theorem, this implies that for all
$T$,
\*[
	\EE_{\baseenv{},\policy{}}[\numobsaction{T}(\refactionval)]
	&\geq \frac{\log(T\Delta\eps)-\log(8C\,\sqrt{\abs{\actionspace} T}\,)}{(1/2)\KL{\bernoullidist(3/4)}{\bernoullidist(3/4+2\Delta(1+\eps))}}
	&\geq \frac{1}{8\Delta^2(1+\eps)^2} \log\frac{\Delta\eps\sqrt{T}}{8C\,\sqrt{\abs{\actionspace}}}.
\]

Finally, we combine this with
\*[
	\regretcustom{\baseenv{}}{\policy{}}{T}
	= \sum_{\refactionval\neq1} \Delta \EE_{\baseenv{},\policy{}}[\numobsaction{T}(\refactionval)],
\]
choose $\eps=1$, and set
\*[
	\Delta = \frac{16C\,\sqrt{\abs{\actionspace}}}{\sqrt{T}}.
\]

\manualendproof

%% file: section-files/main-sections/simulation-details.tex
\section{Simulation Details}\label{sec:simulation-details}

Here we provide more details for the simulations in \cref{sec:simulations}. First, the regret bounds are computed by sampling a new data realization for each horizon $T$ we consider, computing the expected regret (with respect to the data randomness) for this realization, and then averaging this value (i.e., over the algorithm randomness) over $M=300$ realizations.

\CUCB{} and \UCB{} are implemented exactly according to \cref{sec:algos}, \HCUCB{} is implemented exactly according to \cref{alg:hyp-test}, and \CUCBt{} is implemented exactly according to Algorithm~3 of \citet{nair21budget} (including their time-adaptive confidence bound).
For \Corral{}, we use the log-barrier method from Algorithms 1 and 2 of \citet{agarwal17corral} with base algorithms \UCB{} and \CUCB{}. 
We use the prescribed learning rate from their Theorem~5 of 
\*[
	\eta =\frac{1}{40 \cdot \Rr(T) \log T},
\]
where $\Rr(T)$ is an upper bound on the regret of \CUCB{}.
In order to use \UCB{} and \CUCB{} with importance-weighted losses, we implement the epoch-based approach of \citet{arora21corral} along with their Freedman's inequality confidence bound of
\*[
	\sqrt{\frac{4\rho \log t}{\TT_t}} + \frac{4\rho \log t}{3\TT_t}
\]
for the arm means (\UCB{}) and the \postcontextname{} conditional means (\CUCB{}) respectively, where $\rho$ is an upper bound on the importance-weighted losses.

For \Corral{},
it is typical in experiments (e.g., \citep{arora21corral}) to swap out the Freedman's inequality confidence bound for the usual Hoeffding's inequality confidence bound.
However, there are no theoretical guarantees for the algorithm then (due to the importance-weighted losses), and we observed in additional experiments that it is still not adaptive (i.e., it does no better than \UCB{} even in \propertyname{} \envnames{}).